\documentclass[11pt]{article}
\usepackage[letterpaper, margin=1in]{geometry}

\usepackage{tikz}
\usetikzlibrary{positioning, fit, calc, shapes.geometric, arrows.meta, backgrounds,
  shadows, decorations.pathreplacing}
\usetikzlibrary{automata}
\usepackage{pgfplots}
\pgfplotsset{compat=1.18} % Or compat=newest

\usepackage[utf8]{inputenc} % allow utf-8 input
\usepackage[T1]{fontenc}    % use 8-bit T1 fonts
\usepackage{hyperref}       % hyperlinks
\usepackage{url}            % simple URL typesetting
\usepackage{booktabs}       % professional-quality tables
\usepackage{amsfonts}       % blackboard math symbols
\usepackage{nicefrac}       % compact symbols for 1/2, etc.
\usepackage{microtype}      % microtypography
\usepackage{xcolor}         % colors
\usepackage{graphicx}

\usepackage{subcaption}
\usepackage{stmaryrd}

\usepackage{algorithm}
\usepackage{algorithmic}

\usepackage{amsmath}
\usepackage{amssymb}
\usepackage{mathtools}
\usepackage{amsthm}

\usepackage{times}
\usepackage[capitalize,noabbrev]{cleveref}
\usepackage{natbib}

\usepackage{dsfont}
\usepackage[mathscr]{euscript}
\usepackage{expl3}
\usepackage{colortbl}
\usepackage{thmtools}
\usepackage{thm-restate}
\usepackage{multirow}
\usepackage{wrapfig}
\usepackage{lipsum}
\usepackage{accents}

\usepackage{mathpazo}
\usepackage{pifont}% http://ctan.org/pkg/pifont
\newcommand{\cmark}{\ding{51}}%
\newcommand{\xmark}{\ding{55}}%

\usepackage{MnSymbol}
\DeclareMathAlphabet\mathbb{U}{msb}{m}{n}
\usepackage{xpatch}

\def\Kset{\mathbb{K}}

\def\Nset{\mathbb{N}}
\def\Zset{\mathbb{Z}}
\def\Rset{\mathbb{R}}

\def\Hset{\mathbb{H}}

\let\P\undefined

\DeclareMathOperator*{\P}{\mathbb{P}}
\DeclareMathOperator*{\E}{\mathbb E}
\DeclareMathOperator*{\argmax}{argmax}

\DeclareMathOperator{\Ind}{\mathbb{I}}

\DeclarePairedDelimiter{\bracket}{[}{]}

\DeclarePairedDelimiter{\paren}{(}{)}
\DeclarePairedDelimiter{\norm}{\|}{\|}

\ExplSyntaxOn

\tl_const:Nn \c_my_uc_alphabet_tl { ABCDEFGHIJKLMNOPQRSTUVWXYZ }
\tl_const:Nn \c_my_full_alphabet_tl { ABCDEFGHIJKLMNOPQRSTUVWXYZ
  abcdefghijklmnopqrstuvwxyz }

\tl_map_inline:Nn \c_my_uc_alphabet_tl
 { \cs_gset:cpn { c #1 } { \mathcal{#1} } }

\tl_map_inline:Nn \c_my_uc_alphabet_tl
 { \cs_gset:cpn { s #1 } { \mathscr{#1} } }

\tl_map_inline:Nn \c_my_full_alphabet_tl
 {
  \cs_gset:cpn { b #1 } { \mathbf{#1} }
  \cs_gset:cpn { sf #1 } { \mathsf{#1} }
 }

\clist_const:Nn \c_my_greek_alphabet_clist
  {
    alpha, beta, gamma, delta, epsilon, zeta, eta, theta, iota, kappa,
    lambda, mu, nu, xi, pi, rho, sigma, tau, upsilon, phi, chi, psi, omega,
    varepsilon, vartheta, varpi, varrho, varsigma, varphi,
    Gamma, Delta, Theta, Lambda, Xi, Pi, Sigma, Upsilon, Phi, Psi, Omega
  }

\clist_map_inline:Nn \c_my_greek_alphabet_clist
  {
    \str_if_eq:nnTF { #1 } { eta }
      {
        \cs_gset:cpx { boldeta } 
          { \exp_not:N \boldsymbol { \exp_not:c { #1 } } }
      }
      {
        \cs_gset:cpx { b #1 } 
          { \exp_not:N \boldsymbol { \exp_not:c { #1 } } }
      }
  }

\ExplSyntaxOff

\newcommand{\Rad}{\mathfrak R}
\newcommand{\h}{\widehat}

\newcommand{\wt}{\widetilde}
\newcommand{\e}{\epsilon}
\newcommand{\ignore}[1]{}

\newcommand{\AC}{\mathsf{AC}}
\newcommand{\NC}{\mathsf{NC}}
\newcommand{\PNC}{\mathsf{PNC}}
\newcommand{\TC}{\mathsf{TC}}
\newcommand{\FO}{\mathsf{FO}}
\newcommand{\MSO}{\mathsf{MSO}}

\hypersetup{
  breaklinks   = true, %splits links across lines
  colorlinks   = true, %Colours links instead of ugly boxes
  urlcolor     = blue, %Colour 
  linkcolor    = blue, %Colour of internal links
  citecolor   = blue %Colour of citations
}

\usepackage[toc, page, header]{appendix}
\declaretheorem{theorem}
\newtheorem{lemma}[theorem]{Lemma} 
\newtheorem{proposition}[theorem]{Proposition} 
\newtheorem{corollary}[theorem]{Corollary}
\newtheorem{definition}[theorem]{Definition}

\title{Rational Transductors}
\author{Mehryar Mohri\\
  \small{mohri@google.com}\\
  \small{Google Research}}
\date{}

\allowdisplaybreaks

\begin{document}

\maketitle

\tableofcontents
\clearpage

\begin{abstract}
  Standard Transformers excel at semantic modeling but struggle with
  rigid sequential logic and state tracking. Theoretical work
  establishes that self-attention is limited to $\AC^0$ (under hard
  attention) or $\TC^0$ (under soft attention), complexity classes
  that often fail to support robust length generalization on
  sequential problems without intermediate chain-of-thought
  \citep{hahn2020theoretical, hao2022formal, merrill2022saturated}.  In this
  work, we introduce \emph{Rational Transductors}, a dual-stream
  architecture that augments the Transformer with a matrix-valued
  recurrence derived from Weighted Finite Automata (WFA).  By
  injecting rational state information into the attention mechanism
  via a \emph{Deep Rational Injection} scheme, our framework strictly
  generalizes the expressive power of Transformers to capture all
  Regular Languages, $\NC^1$-complete languages (such as the word
  problem of $S_5$), and fundamental separations like Parity and
  Modular Counting, while preserving $O(\log T)$ parallel training
  efficiency.  We ground the architecture in a rigorous learning
  theory: we prove that \emph{Random Rational Features} act as a
  universal basis for sequential dependencies, justifying our
  initialization strategy, while establishing that the
  \emph{Differentiable Rational Feature} regime is necessary to close
  the representational compactness gap.  Theoretical analysis and
  empirical results demonstrate that Rational Transductors solve the
  ``Regular Gap,'' enabling robust length generalization on algorithmic
  tasks where standard Transformers fail, without the sequential
  computational bottlenecks of traditional RNNs.
\end{abstract}

\section{Introduction}

The Transformer architecture \citep{vaswani2017attention} has
revolutionized sequence modeling, establishing itself as the de facto
standard for natural language processing, code generation, and beyond
\citep{brown2020language, touvron2023llama}.
Its success is largely attributed to the self-attention mechanism
\citep{schmidhuber1992learning,graves2013generating,bahdanau2014neural,
  luong2015effective}, which models long-range semantic dependencies
by allowing every token to interact directly with every other token.
However, this semantic power comes with a well-documented blind spot:
standard Transformers struggle with rigid \emph{sequential logic} and
\emph{state tracking}
\citep{liu2022transformers, bhattamishra2020ability}. Theoretical analyses have shown that
self-attention---without intermediate recurrence or
chain-of-thought---is limited to $\AC^0$ (under hard attention)
\citep{hao2022formal} or $\TC^0$ (under soft attention)
\citep{merrill2023logic, chiang2024transformers}, complexity classes that struggle to
represent unbounded sequential dependencies uniformly
\citep{huang2025formal,merrill2022saturated}.  While $\TC^0$ models
can theoretically approximate tasks like parity, they lack the
inductive bias to learn state-tracking solutions that generalize to
unseen lengths. Specifically, standard Transformers often fail to
\emph{learn} robust solutions for tasks outside the C-RASP fragment
\citep{yang2025kneedeep,huang2025formal}, such as modular counting. In
practice, this manifests as brittleness: models trained on short
contexts frequently fail to maintain consistent state (e.g., tracking
variable values or nested brackets) when deployed on longer sequences
\citep{anil2022exploring}.

To address these limitations, a resurgence of interest in Recurrent
Neural Networks (RNNs) and State Space Models (SSMs) has emerged
\citep{gu2021efficiently, gu2023mamba, smith2022simplified}.  These
architectures reintroduce a latent state that evolves over time,
theoretically enabling infinite context tracking.  However, they often
face a structural trade-off: simple time-invariant recurrences (like
S4) are efficient but lack expressivity. While recent selective state
space models (like Mamba) introduce token-dependence to bridge this
gap, they typically do so by interleaving recurrence into the deep
backbone, which reintroduces layer-wise sequential dependencies during
inference or training. In contrast, simple gated RNNs (like LSTM)
remain fundamentally sequential.  Unlike interleaved architectures
that incur an $O(L \log T)$ training bottleneck due to sequential
layer-wise recurrences, Rational Transductors decouple state-tracking
from feature mixing. This allows the entire rational state history to
be pre-computed via a single parallel scan, reducing the total
parallel depth to $O(L + \log T)$.

In this work, we argue that the dichotomy between "Attention"
(semantics) and "Recurrence" (syntax) is a false one.  We introduce
\emph{Rational Transductors}, a dual-stream architecture that unifies
the semantic flexibility of Transformers with the rigorous
state-tracking capabilities of Weighted Finite Automata (WFA).  Our
approach is grounded in the formal theory of Rational Power Series
\citep{Schutzenberger1961}, which provides the mathematical foundation
for regular languages and their quantitative generalizations.

We argue that the failure of Transformers to generalize on algorithmic
tasks is not due to a lack of capacity, but a lack of \emph{syntactic
  inductive bias}.  To correct this, we augment the Transformer with a
\emph{Rational Feature Head}, a matrix-valued recurrence that acts as
a dedicated co-processor for sequential logic.  Crucially, unlike
standard RNNs that use non-linear activations ($\tanh$, sigmoid), our
rational states evolve via linear matrix multiplication.
This design choice yields two decisive advantages:
\begin{enumerate}

\item Parallel Scalability: The linear recurrence
  $h_t = \sfM_{x_t} h_{t-1}$ can be computed via parallel associative
  scans (prefix sums) in $O(\log T)$ time \citep{blelloch1990prefix,ladner1980parallel},
  bypassing the sequential bottleneck that plagues traditional RNNs.

\item Theoretical Transparency: The state dynamics correspond exactly
  to Weighted Finite Automata \citep{mohri2009weighted,droste2009handbook}, allowing us
  to leverage decades of formal language theory, such as the
  Schützenberger representation theorem \citep{Schutzenberger1961} and
  Fliess' Theorem \citep{Fliess1974}, to prove guarantees on
  expressivity and generalization that are impossible for black-box
  RNNs.

\end{enumerate}

We ground this architecture in a rigorous learning theory
(Section~\ref{sec:learning_theory}).  First, analyzing the
\emph{Random Rational Feature} limit
(Section~\ref{subsec:random_universality}), we prove that a
sufficiently wide, randomly initialized rational head acts as a
universal basis for sequential dependencies.  This theoretical result
justifies our use of \emph{near-identity initialization} to capture
long-term context from the start of training.  However, we show that
random features are exponentially inefficient for precise algorithmic
tasks.  Therefore, our primary contribution is the
\emph{Differentiable Rational Feature} regime, where the transition
matrices are learned end-to-end.  We prove that this learned regime
strictly expands the expressivity of Transformers to capture all
Regular Languages, $\NC^1$-complete languages (such as the word problem
of $S_5$), and fundamental separations like Parity and Modular
Counting, while maintaining numerical stability through novel spectral
parameterizations (Section~\ref{subsec:optimization_dynamics}).

\textbf{Related Work.} Our work lies at the intersection of three
active research streams:
\begin{itemize}

\item \textbf{State Space Models (SSMs) and Linear RNNs}: Recent
  advances in efficient sequence modeling, including S4
  \citep{gu2021efficiently}, Mamba \citep{gu2023mamba}, RWKV
  \citep{peng2023rwkv}, DeltaNet \citep{schlag2021linear}, and Kimi
  Linear \citep{kimilinear2025}, rely on linearizing state updates to
  enable parallel training. While these architectures are often
  grounded in signal processing, our framework provides a
  complementary automata-theoretic characterization. Although existing
  literature has established complexity bounds and formal language
  mappings for certain linear recurrences
  \citep{merrill2022saturated}, the Rational Transductor framework
  formalizes this class by providing a rigorous algebraic completion
  via the Krohn-Rhodes decomposition. This allows us to prove that
  such linear recurrences are sufficient to solve the "Regular Gap"
  and recognize $\NC^1$-complete languages while maintaining a decoupled
  sidecar design. Furthermore, unlike "Deep SSMs" (e.g., Mamba, H3)
  which interleave recurrence and mixing layers—thereby reintroducing
  layer-wise sequential dependencies during training—Rational
  Transductors adopt a "Sidecar" design. By keeping the recurrent
  state evolution strictly input-driven, we decouple state tracking
  (WFA) from feature mixing (Transformer), guaranteeing an optimal
  $O(L + \log T)$ parallel depth without the iterative approximations
  or serialization inherent to stacked SSMs.
  We note that recent work on structured sparse transition matrices
  \citep{pdssm2025} pursues a complementary approach to state tracking
  in SSMs, while stack-augmented Transformers
  \citep{dusell2024stack} explore pushdown extensions for context-free
  capabilities; our framework differs by providing
  automata-theoretic guarantees for regular languages via a sidecar
  linear recurrence.  The interleaved vs.\ sidecar trade-off is further
  analyzed in Appendix~\ref{subsec:architectural_extensions}:
  interleaving yields higher per-parameter expressivity via non-linear
  cascades (Theorem~\ref{th:nonlinear_irreducibility}), but
  reintroduces a sequential training bottleneck of depth
  $O(L \log T)$.

\item \textbf{Expressivity of Transformers:}
  \citet{hahn2020theoretical} first showed that hard-attention
  Transformers cannot recognize Parity or Dyck languages, and that
  soft-attention Transformers cannot do so robustly.
  \citet{hao2022formal} and \citet{merrill2023logic} then
  established the $\AC^0$ and $\TC^0$ upper bounds for Transformers
  (depending on attention hardness), later strengthened by
  \citet{chiang2024transformers}, highlighting their inability to
  robustly model sequential state.  Our work provides a constructive
  proof that augmenting attention with linear recurrence is sufficient
  to break this barrier and recognize $\NC^1$-complete languages.

\item \textbf{Spectral Learning of Automata:} We draw inspiration from
  spectral learning algorithms for WFAs
  \citep{BalleMohri2015,balle2012spectral,hsu2009spectral},
  effectively embedding a spectral extraction mechanism directly into
  the deep learning optimization loop.
\end{itemize}

By bridging the gap between deep learning and automata theory,
Rational Transductors offer a principled path toward foundation models
that are not only semantically fluent but also syntactically robust.

\textbf{Paper Organization.}  The remainder of this paper is organized
as follows.  Section~\ref{sec:rational_features} formally defines the
Rational Transductor architecture, detailing the WFA formalism, the
structured parameterization of transitions, and the Deep Rational
Injection method.  Section~\ref{sec:motivation} motivates the
architecture by analyzing the complementary strengths of attention and
rational recurrences.  Section~\ref{sec:expressivity} provides a
rigorous theoretical analysis of the model's expressivity in the
learnable regime, proving it solves the "Regular Gap."
Section~\ref{sec:learning_theory} analyzes the learning theory of the
model, establishing results on the universality of random features,
optimization stability, and generalization bounds.
Section~\ref{sec:training_recipe} outlines the concrete training
recipe, including the parallel scan algorithm and stability-enforcing
normalizations.  Section~\ref{sec:experiments} presents empirical
validation, and Section~\ref{app:theoretical_background} reviews
foundational theory.

\section{Rational Features Framework}
\label{sec:rational_features}

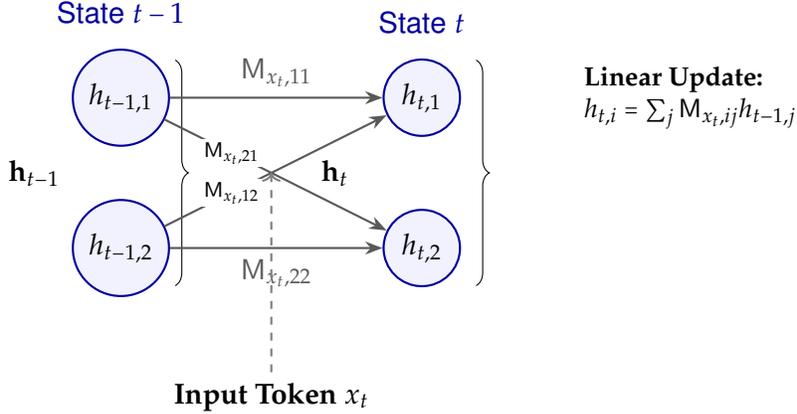
\begin{figure}[t]
    \centering
        \begin{tikzpicture}[
        scale=1.0, transform shape,
        font=\sffamily,
        state_node/.style={circle, draw=blue!60!black, fill=blue!5, thick, minimum size=1.0cm},
        transition/.style={->, >=Stealth, thick, color=gray!70!black},
        weight_label/.style={font=\scriptsize, color=black, fill=white, inner sep=1pt}
    ]

    % --- TIME T-1 ---
    \node (h_prev_1) [state_node] at (0, 2) {$h_{t-1, 1}$};
    \node (h_prev_2) [state_node] at (0, 0) {$h_{t-1, 2}$};
    \node [above=0.2cm of h_prev_1, color=blue!60!black] {State $t-1$};
    % Vector Brace
    \draw [decorate,decoration={brace,amplitude=5pt,mirror,raise=4ex}]
      (0,-0.5) -- (0,2.5) node[midway,xshift=-3em] {$\bh_{t-1}$};
    % --- TIME T ---
    \node (h_curr_1) [state_node] at (4, 2) {$h_{t, 1}$};
    \node (h_curr_2) [state_node] at (4, 0) {$h_{t, 2}$};
    \node [above=0.2cm of h_curr_1, color=blue!60!black] {State $t$};
    % Vector Brace
    \draw [decorate,decoration={brace,amplitude=5pt,mirror,raise=4ex}]
      (4,-0.5) -- (4,2.5) node[midway,xshift=-3em] {$\bh_{t}$};
    % --- TRANSITIONS ---
    % \sfM_{x_t, 11}
    \draw[transition] (h_prev_1) -- (h_curr_1) 
        node[midway, above] {$\sfM_{x_t, 11}$};
    % \sfM_{x_t, 22}
    \draw[transition] (h_prev_2) -- (h_curr_2) 
        node[midway, below] {$\sfM_{x_t, 22}$};
    % \sfM_{x_t, 21} (Cross)
    \draw[transition] (h_prev_2) -- (h_curr_1) 
        node[midway, pos=0.3, weight_label] {$\sfM_{x_t, 12}$};
    % \sfM_{x_t, 12} (Cross)
    \draw[transition] (h_prev_1) -- (h_curr_2) 
        node[midway, pos=0.3, weight_label] {$\sfM_{x_t, 21}$};
    % --- ANNOTATION ---
    \node [right=1.5cm of h_curr_1, align=left, font=\small] (formula) {
        \textbf{Linear Update:}\\
        $h_{t,i} = \sum_{j} \sfM_{x_t, ij} h_{t-1,j}$
    };
    % Input Token Label
    \node [below=1.0cm of h_prev_2, xshift=2cm, font=\bfseries] (input) {Input Token $x_t$};
    \draw[->, dashed, thick, color=gray] (input) -- (2, 1);

    \end{tikzpicture}
    \caption{Visualizing the Rational State Update.  The hidden state
      vector $\bh_t$ (right) is computed as a linear transformation of
      the previous state $\bh_{t-1}$ (left).  Each component $h_{t,i}$
      aggregates the weighted paths from the previous step,
      illustrating the "sum of paths" definition.}
    \label{fig:wfa_hidden_states}
\end{figure}

\subsection{Weighted Automata}

We view state tracking through the lens of Weighted Finite Automata
(WFA). Formally, a WFA over the field of real numbers $\Rset$ is
defined as a tuple 
$\sA = (\Sigma, d, \balpha, \{\sfM_\sigma\}_{\sigma \in \Sigma})$, where:
\begin{itemize}
\item $\Sigma$ is the finite alphabet of tokens.
\item $d \in \Nset$ is the dimension of the state space (number of states).
\item $\balpha \in \Rset^d$ is the initial state vector.
\item $\sfM_\sigma \in \Rset^{d \times d}$ is the transition matrix
  associated with token $\sigma$.
\end{itemize}
Given a sequence of input tokens $x = (x_1, \dots, x_T)$, a standard
WFA computes a scalar value. However, for the purpose of feature
extraction, we are interested in the sequence of vectors
$\bh_t \in \Rset^d$ (hidden states) produced by the automaton, which
is defined as follows:
\begin{equation}
  \label{eq:linear-recurrence}
  \bh_t = \sfM_{x_t} \bh_{t - 1}, \quad \text{with } \bh_0 = \balpha.
\end{equation}
The $i$-th component $h_{t, i}$ represents the sum of the weights of
all paths labeled with the prefix $x_{1:t}$ that end at state $i$,
weighted by the initial values in $\balpha$ (see
Figure~\ref{fig:wfa_hidden_states}).
We extend the definition of the matrices $\sfM_\sigma$ to sequences
using the shorthand $\sfM_x = \sfM_{x_T} \cdots \sfM_{x_1}$, which
allows us to write the state after processing the full sequence as
$\bh_T = \sfM_x \balpha$.

\textbf{Omission of $\bbeta$.} In classical automata theory, a WFA
also includes a final weight vector $\bbeta \in \Rset^d$ to map the
final state to a scalar output. We omit $\bbeta$ in our definition as
we are only interested in the sequence of intermediate vectors in
$\Rset^d$ (hidden states).

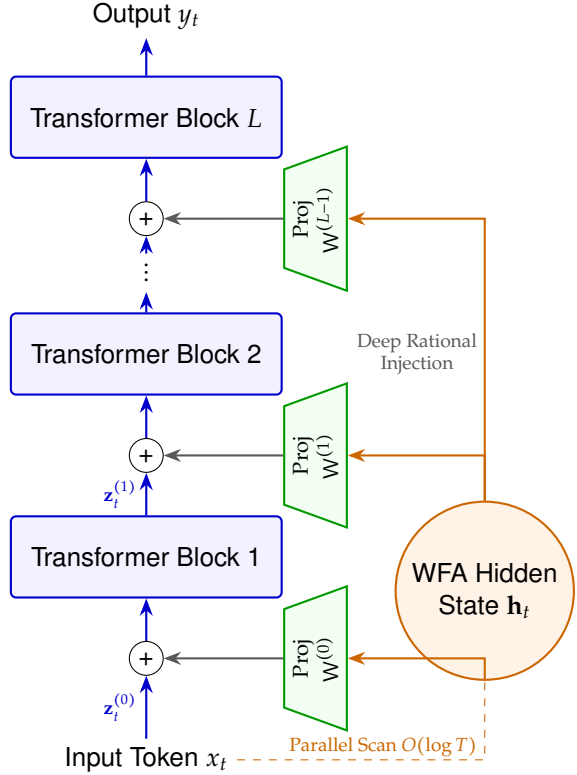
\begin{figure}[t]
    \centering
        \begin{tikzpicture}[
        scale=0.9, transform shape,
        font=\sffamily,
        % Style for the Transformer Blocks
        tf_block/.style={
            rectangle, 
            draw=blue!80!black, 
            fill=blue!5, 
            thick, 
            minimum width=4cm, 
            minimum height=1.2cm, 
            rounded corners=2pt,
            align=center
        },
        % Style for the WFA Head (Circle)
        wfa_node/.style={
            circle, 
            draw=orange!80!black, 
            fill=orange!10, 
            thick, 
            minimum size=2.0cm, 
            align=center
        },
        % Style for Projections (Trapezoids or Boxes)
        proj_node/.style={
            trapezium, 
            trapezium angle=70, 
            draw=green!60!black, 
            fill=green!5, 
            thick, 
            minimum height=0.8cm, 
            text width=1.2cm, 
            align=center,
            font=\footnotesize
        },
        % Style for Summation
        sum_node/.style={
            circle, 
            draw, 
            fill=white, 
            inner sep=1pt, 
            minimum size=0.5cm
        },
        % Arrow Styles
        main_stream/.style={->, >=Stealth, thick, blue!80!black},
        rational_stream/.style={->, >=Stealth, thick, orange!80!black},
        injection_arrow/.style={->, >=Stealth, thick, gray!70!black}
    ]

    % --- NODES ---

    % Input
    \node (input) at (0, 0) {Input Token $x_t$};
    % Layer 1 Input Sum
    \node[sum_node] (sum0) at (0, 1.5) {+};
    \node[tf_block] (block1) at (0, 3) {Transformer Block 1};
    
    % Layer 2 Input Sum
    \node[sum_node] (sum1) at (0, 4.5) {+};
    \node[tf_block] (block2) at (0, 6) {Transformer Block 2};
    
    % Dots for depth
    \node (dots) at (0, 7.2) {$\vdots$};
    % Layer L Input Sum
    \node[sum_node] (sumL_minus_1) at (0, 8.0) {+};
    \node[tf_block] (blockL) at (0, 9.5) {Transformer Block $L$};
    
    % Output
    \node (output) at (0, 11) {Output $y_t$};
    % WFA Head
    \node[wfa_node] (wfa) at (5, 2.5) {WFA Hidden\\ State $\bh_t$};
    % Projections (Corrected Indices)
    \node[proj_node, rotate=90] (proj0) at (2.5, 1.5) {Proj\\$\sfW^{(0)}$};
    \node[proj_node, rotate=90] (proj1) at (2.5, 4.5) {Proj\\$\sfW^{(1)}$};
    \node[proj_node, rotate=90] (projL_minus_1) at (2.5, 8.0) {Proj\\$\sfW^{(L-1)}$};
    % --- CONNECTIONS ---

    % Main Stream
    \draw[main_stream] (input) -- node[left, font=\scriptsize] {$\bz_t^{(0)}$} (sum0);
    \draw[main_stream] (sum0) -- (block1.south);
    
    \draw[main_stream] (block1.north) -- node[left, font=\scriptsize] {$\bz_t^{(1)}$} (sum1);
    \draw[main_stream] (sum1) -- (block2.south);
    
    \draw[main_stream] (block2.north) -- (dots.south);
    \draw[main_stream] (dots.north) -- (sumL_minus_1);
    \draw[main_stream] (sumL_minus_1) -- (blockL.south);
    
    \draw[main_stream] (blockL.north) -- (output);
    % Rational Stream (Injection)
    \draw[rational_stream] (wfa) |- (proj0.south);
    \draw[rational_stream] (wfa) |- (proj1.south);
    \draw[rational_stream] (wfa) |- (projL_minus_1.south);
    % Connecting Projections to Sums
    \draw[injection_arrow] (proj0.north) -- (sum0);
    \draw[injection_arrow] (proj1.north) -- (sum1);
    \draw[injection_arrow] (projL_minus_1.north) -- (sumL_minus_1);
    % Parallel Scan Line
    \draw[dashed, orange!80!black] (input.east) -- (5, 0) -- (wfa.south);
    \node[text=orange!80!black, font=\scriptsize] at (3.5, 0.2) {Parallel Scan $O(\log T)$};

    % Labels
    \node[text=gray!70!black, font=\scriptsize, align=center] at (4.0, 6) {Deep Rational\\Injection};
    \end{tikzpicture}
    \caption{The Rational Transductor Architecture.  The Rational Head
      extracts state variables $\bh_t$. These states are injected into
      the Attention Stream via layer-specific projections
      $\sfW^{(l)}$, augmenting the semantic hidden states
      $\bz_t^{(l)}$.}
    \label{fig:rt_arch}
  \end{figure}

\subsection{Rational Feature Layers}
\label{sec:rational_feature_layers}

We define the \emph{rational feature vector} $\bh_t$ at time step $t$
as the forward state of the automaton after processing the prefix
$x_{1:t}$, that is $\bh_t = \sfM_{x_{1:t}} \balpha$. The state evolves
according to the linear recurrence \eqref{eq:linear-recurrence}.

Unlike standard Recurrent Neural Networks (RNNs) which use non-linear
activation functions (e.g., $\tanh$ or $\sigma$), the update in
\eqref{eq:linear-recurrence} is linear. This linearity is the
defining characteristic of rational series and provides two distinct
advantages:
\begin{enumerate}

\item Theoretical Clarity: It guarantees that the features
  capture Regular Languages (and their weighted generalizations)
  exactly.

\item Parallel Scalability: The recurrence
  $\bh_t = \sfM_{x_t} \dots \sfM_{x_1} \balpha$ can be computed
  efficiently on modern hardware using parallel associative scans
  (prefix sums), avoiding the sequential bottleneck of standard RNNs.
  While the constant factor depends on matrix multiplication costs
  (potentially $O(d^3)$ or $O(d^2)$ depending on structure), we assume
  moderate $d$ such that this overhead is negligible compared to the
  quadratic cost of attention.
\end{enumerate}

\subsection{Parameterization of Transition Matrices}
\label{subsec:parameterization}

The Rational Transductor framework is agnostic to the specific
internal structure of the transition matrices $\sfM_\sigma$.
Examples include diagonal, permutation, orthogonal via Cayley
parametrization, or diagonal-plus-low-rank matrices with fixed rank.
While our experiments focus on two primary regimes (decay and
conservation), the architecture supports any structured matrix family
$\cS \subset \Rset^{d \times d}$:

\paragraph{Diagonal plus Low-Rank (DPLR).} 
For general sequence modeling tasks requiring efficient mixing and
fading memory, we adopt the structure:
\begin{equation}
    \sfM_\sigma = \sfD_\sigma + \sfU_\sigma \sfV_\sigma^\top,
\end{equation}
where $\sfD_\sigma$ is diagonal and
$\sfU_\sigma, \sfV_\sigma \in \Rset^{d \times r}$ are low-rank
factors. While the DPLR structure reduces the cost of serial unrolling
and inference to $O(dr)$, parallel training via associative scans
necessitates dense matrix-matrix multiplication. Consequently, the
training recurrence scales as $O(d^3)$. This remains computationally
efficient for the moderate state dimensions ($d \in [4, 32]$) used
in this work, where the overhead is negligible compared to the $O(T)$
costs of surrounding layers. Similar low-rank updates have been
successfully deployed in other linear RNN architectures, such as
DeltaNet, to approximate gradients efficiently.

\paragraph{Orthogonal Parameterization (Cayley).} 
For algorithmic tasks requiring infinite memory conservation (e.g.,
modular counting), we parameterize $\sfM_\sigma$ to be strictly
orthogonal ($\sfM_\sigma^\top \sfM_\sigma = \sfI$). In practice, this
is achieved via the \emph{Cayley transform}. Let
$\sfW_\sigma \in \Rset^{d \times d}$ be a learnable parameter
matrix. We construct a skew-symmetric matrix
$\sfA_\sigma = \sfW_\sigma - \sfW_\sigma^\top$ and define the
transition as:
\begin{equation}
    \sfM_\sigma = (\sfI + \sfA_\sigma)(\sfI - \sfA_\sigma)^{-1}.
\end{equation}
Since the Cayley transform maps skew-symmetric matrices to the special
orthogonal group $SO(d)$, this guarantees that the state norm
$\|\bh_t\|_2$ is preserved exactly over infinite horizons, regardless
of the input sequence. We discuss a generalized version of this, the
\emph{Unified Scaled Cayley Parameterization}, in
Section~\ref{subsec:spectral_control}, which augments this form to
support both exact conservation and learnable decay within a single
differentiable framework.

\paragraph{Other Structured Families.}
We note that other structured parameterizations such as Butterfly
matrices (for efficient FFT-like mixing), Toeplitz or Hankel matrices
(for convolutional structures), or Block-Sparse matrices can be
seamlessly substituted into the Rational Head to induce different
inductive biases without altering the fundamental architecture or
training recipe.

\paragraph{Shared Basis Parameterization.}
To further reduce parameters and improve generalization, we can
express transitions as
\[
    \sfM_\sigma = \sum_{k=1}^K a_{\sigma,k} \sfB_k,
\]
where $\{\sfB_k\}_{k=1}^K$ are shared basis matrices and
$a_{\sigma,k}$ are token-dependent coefficients. This reduces
parameters, improves generalization, and connects to tensor
factorization \citep{yang2017tensor}.  This parameterization induces a low-rank tensor
factorization over $(\sigma, i, j)$ and allows the model to learn
token-conditioned dynamics in a shared latent basis. Empirically,
this sharing acts as a strong inductive bias that stabilizes training
when $|\Sigma|$ is large.

\paragraph{Parallel Combination (Sum).}
We can construct a \emph{Mixed Rational Head} by running multiple
independent automata in parallel. Mathematically, this corresponds to
the \emph{Direct Sum} of the transition matrices:
\begin{equation}
    \sfM_\sigma = \sfM_\sigma^{(1)} \oplus \sfM_\sigma^{(2)} = 
    \begin{pmatrix} 
        \sfM_\sigma^{(1)} & 0 \\ 
        0 & \sfM_\sigma^{(2)} 
    \end{pmatrix}.
\end{equation}
This formulation allows the model to instantiate parallel heads with
distinct dynamical biases (e.g., mixing a DPLR head for fading context
with an Orthogonal head for exact counting). The resulting state
dimension is additive ($d = d_1 + d_2$), maintaining efficiency.

\begin{figure}[t]
    \centering
        \begin{tikzpicture}[
        font=\sffamily,
        head/.style={rectangle, draw=black, thick, minimum width=2.5cm, minimum height=1.5cm, rounded corners, align=center},
        concat/.style={rectangle, draw=black, fill=gray!20, minimum width=0.8cm, minimum height=3.5cm, rounded corners},
        label_node/.style={font=\scriptsize\bfseries, align=center}
    ]

    % Input
    \node (input) at (0, 0) {Input $x_t$};

    % --- HEAD 1: ORTHOGONAL (Top) ---
    \node[head, fill=blue!10] (head1) at (4, 1.8) {};
    \node[label_node, above] at (head1.north) {Head 1: Orthogonal};
    % Visual
    \draw[->, thick, blue] (3.6, 1.95) arc (180:-135:0.45cm);
    \node at (4, 1.8) {$\gamma=1$};
    \node[font=\tiny] at (4, 1.4) {(Counting/Memory)};

    % --- HEAD 2: STOCHASTIC (Bottom) ---
    \node[head, fill=green!10] (head2) at (4, -1.8) {};
    \node[label_node, above] at (head2.north) {Head 2: Stochastic};
    % Visual
    \node[circle, fill=black, inner sep=1pt] (s1) at (3.5, -1.8) {};
    \node[circle, fill=black, inner sep=1pt] (s2) at (4.5, -1.8) {};
    \draw[->] (s1) to[bend left] (s2);
    \draw[->] (s2) to[bend left] (s1);
    \node[font=\tiny] at (4, -2.2) {(Switching/Logic)};

    % --- CONCAT BLOCK ---
    \node[concat] (concat) at (7.5, 0) {};
    \node[rotate=90] at (concat.center) {\textbf{Concat}};

    % --- ROUTING ARROWS ---
    % Split Input
    \draw[->, thick] (input) -- (1.5, 0) |- (head1.west);
    \draw[->, thick] (input) -- (1.5, 0) |- (head2.west);

    % Outputs to Concat (Entering cleanly from the left)
    \draw[->, thick] (head1.east) -- (concat.west |- head1.east);
    \draw[->, thick] (head2.east) -- (concat.west |- head2.east);

    % Final Output
    \node (output) at (9.5, 0) {$\mathbf{h}_t$};
    \draw[->, thick] (concat.east) -- (output);

    % --- ANNOTATION ---
    \node[align=center, font=\small, color=gray!80!black] at (4.5, -3.5) {
        \textbf{Algebraic Interpretation:}\\
        $h_t = h^{(1)}_t \oplus h^{(2)}_t$ (Direct Sum)
    };

    \end{tikzpicture}
    \caption{The Universal Rational Transductor.
    The architecture instantiates parallel heads with distinct dynamical biases:
    \emph{Orthogonal} (top) for infinite memory and
    \emph{Stochastic} (bottom) for discrete switching.
    These independent features are concatenated, corresponding to the
    direct sum ($\oplus$) of the underlying automata.}
    \label{fig:universal_transductor}
\end{figure}
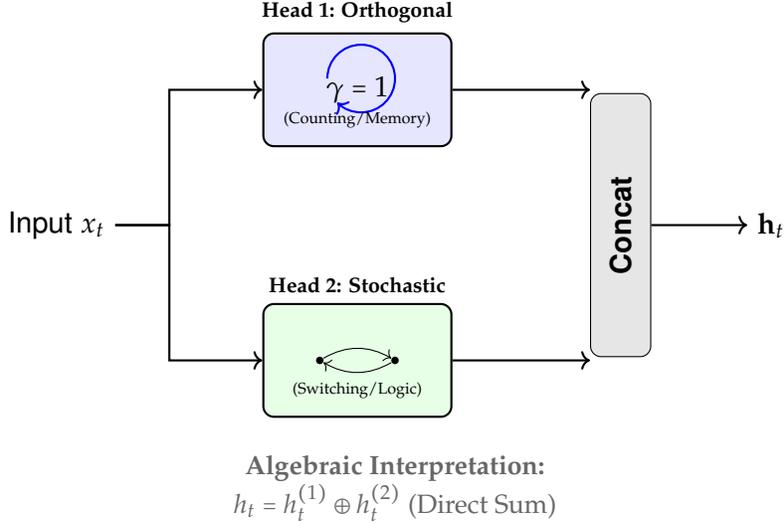

\paragraph{Mixture of Dynamics (Universal Transductor).}
To enable a general-purpose architecture that requires no manual
tuning, we propose the \emph{Universal Rational Transductor}
(Figure~\ref{fig:universal_transductor}). Inspired by Multi-Head
Attention, we instantiate parallel rational heads (via the Direct Sum
above) with distinct dynamical properties:
\begin{itemize}
    \item \textbf{Orthogonal Heads (Cayley):} Learn strictly norm-preserving 
    rotations, ideal for continuous group operations and infinite memory 
    tasks (e.g., counting, parity).
    
    \item \textbf{Stochastic Heads (Simplex):} Parameterized as 
    column-stochastic matrices, these heads function as differentiable 
    finite state automata, ideal for discrete logic, state switching, 
    and "hard resets" (e.g., addition, regular languages).
\end{itemize}
The outputs of these diverse kernels are concatenated and projected, 
allowing the network to autonomously route signals to the dynamical 
system best suited for the underlying algorithm.
While recent works like PD-SSM \citep{pdssm2025} have empirically
adopted diagonal/permutation structures, our Universal Transductor
derives this architecture from \emph{first principles} via the
Krohn-Rhodes decomposition (Theorem \ref{th:krohn_rhodes}).  This
proves that the mixture of orthogonal (group) and stochastic
(aperiodic) heads is not merely an efficient heuristic, but the
\emph{algebraically necessary} structure for capturing the full
hierarchy of regular languages.

\paragraph{Note on Cross-Products.}
While the parallel combination implements the \emph{sum} of rational
series, one can also implement the \emph{Cross-Product} (or Tensor
Product) $\sfM_\sigma = \sfM_\sigma^{(1)} \otimes \sfM_\sigma^{(2)}$
to compute the \emph{Hadamard product} (intersection) of two
series. However, as this scales quadratically ($d_1 \times d_2$), we
typically rely on the downstream attention layers to compute these
multiplicative interactions via the "Virtual Tensorization" mechanism
(Section~\ref{prop:virtual_tensorization}).

\paragraph{Computational Efficiency and Associativity.}  To use the
parallel prefix scan algorithm
(Section~\ref{sec:rational_feature_layers}), the family of matrices
$\cM$ must be closed under multiplication.  We explicitly note that
the scan operation requires the associative aggregation of the
\emph{dense} operators $\sfM_t$. We do not apply low-rank projections
\emph{during} the scan, as this would violate associativity.  Instead,
we maintain computational efficiency by restricting the rational state
dimension $d$ to be small (typically $d \in [4, 32]$). Under this
constraint, the $O(d^3)$ cost of matrix-matrix multiplication at each
step is negligible compared to the $O(T)$ costs of the surrounding
layers.

\subsection{Deep Integration of Rational Features}
\label{sec:deep_injection}

A naive approach to integration would be to simply concatenate the
rational feature vector $\bh_t$ to the input token embedding (Input
Injection). While standard residual connections theoretically allow
information to propagate to depth $L$, relying on them forces the
Transformer to preserve the exact state $\bh_t$ within the semantic
backbone $\bz_t$, competing with feature extraction and suffering from
signal attenuation due to repeated layer normalizations.  To address
this, we propose \emph{Deep Rational Injection}. Instead of augmenting
only the input, we inject the rational features directly into the
hidden state of every Transformer block via an independent
pathway. This ensures that a fresh, uncorrupted view of the precise
state tracking information is available at all levels of abstraction.

Let $\bz_t^{(l)} \in \Rset^{d_{\text{model}}}$ denote the
Transformer's hidden representation at time step $t$ immediately
before layer $l$ (where $l=0$ represents the initial token
embeddings). We modify the input to each layer $l$ by adding a
projected view of the rational state:
\begin{equation}
  \wt \bz_t^{(l)}
  = \bz_t^{(l)} + \sfW_{\text{proj}}^{(l)} \bh_t
\end{equation}
where $\sfW_{\text{proj}}^{(l)} \in \Rset^{d_{\text{model}} \times d}$
is a learnable linear projection unique to layer $l$.  By using a
layer-specific projection $\sfW_{\text{proj}}^{(l)}$, the model can
extract different aspects of the state history relevant to different
depths of processing. For instance, early layers might use $\bh_t$ for
local syntactic parsing (e.g., matching parentheses), while deeper
layers might use the same $\bh_t$ to resolve long-term dependencies
(e.g., tracking subject-verb agreement across long clauses). The
modified stream $\wt \bz_t^{(l)}$ is processed by the standard
Self-Attention and Feed-Forward sub-layers, and the result is added
residually to the original backbone:
\begin{equation}
\bz_{t}^{(l+1)} = \text{TransformerLayer}_l(\wt \bz_t^{(l)}).
\end{equation}
We use the standard Pre-Norm configuration, where
$\text{TransformerLayer}(\bx) = \bx + \text{Attn}(\text{LN}(\bx)) +
\text{FFN}(\text{LN}(\bx))$, ensuring that the injected rational state
gradients propagate directly without identifying vanishing through
normalization layers.
Crucially, because $\bh_t$ is computed via a parallel scan, this deep
injection adds no sequential dependency to the Transformer, preserving
the $O(L + \log T)$ parallel time complexity. This defines our models,
\emph{rational transductors} (RTs). Figure~\ref{fig:rt_arch}
illustrates their architecture.

\subsection{Architectural Extensions}
\label{subsec:architectural_extensions}

\paragraph{Stacked Rational Transductors.}
While the canonical Rational Transductor architecture
(Figure~\ref{fig:rt_arch}) uses a single rational head broadcast to
all layers, the framework naturally admits a stacked generalization.
A \emph{Stacked Rational Transductor} of depth $K$ consists of a
sequence of $K$ Rational Transductor blocks, where the output stream
of the $k$-th block serves as the input to the $(k + 1)$-th block.
Formally, let $\bu_t^{(k-1)}$ be the input vector to block $k$
at time $t$. The block computes a local rational state $\bh_t^{(k)}$
and updates the residual stream:
\begin{align}
  \bh_t^{(k)} & = \sfM_{x_t}^{(k)} \bh_{t-1}^{(k)} + \sfV^{(k)} \bu_t^{(k-1)}
                \label{eq:stack_recurrence} \\
  \bu_t^{(k)} & = \text{TransformerBlock}_k\paren*{ \bu_t^{(k-1)}
                + \sfW_{\text{proj}}^{(k)} \bh_t^{(k)} }
\end{align}
where $\sfM^{(k)}$ is the transition logic specific to layer $k$.

\paragraph{The Linear Collapse Property.}
A crucial theoretical observation guides our preference for the
single-head design over the naive stack. In the regime where the
inter-layer dependence is mediated by a linear map applied to the
previous block’s output (and no non-linearity is applied to the
recurrent state), the cascade of WFAs is reducible.
\begin{proposition}[Reducibility of Cascaded WFAs]
  A cascade of $K$ linear Weighted Finite Automata, where the state of
  automaton $k$ depends linearly on the state of automaton $k-1$, is
  algebraically equivalent to a single WFA with a larger state space
  dimension $d_{\text{total}} = \sum_{k=1}^K d_k$.
\end{proposition}
\begin{proof}
  Consider two stacked states $\bh^{(1)}$ and $\bh^{(2)}$. The joint
  system update can be written as a block triangular matrix:
    \begin{equation}
        \begin{pmatrix} \bh_t^{(1)} \\ \bh_t^{(2)} \end{pmatrix}
        =
        \begin{pmatrix}
            \sfM^{(1)} & 0 \\
            \sfW_{\text{inter}} & \sfM^{(2)}
        \end{pmatrix}
        \begin{pmatrix} \bh_{t-1}^{(1)} 
\\ \bh_{t-1}^{(2)} \end{pmatrix}
    \end{equation}
    This block matrix defines a valid transition matrix
    $\sfM_{\text{joint}}$ for a single WFA. Thus, stacking linear
    recurrences does not strictly expand the class of representable
    functions beyond $\sT_{\text{Rat}}$; it merely structures the
    transition matrix.
\end{proof}

Despite this algebraic reducibility, there is a
representation-theoretic benefit to the stacked parameterization. By
enforcing the block-triangular structure inherent in the cascade, the
model uses significantly fewer parameters to represent the same total
state dimension, acting as a strong inductive bias for decomposable
processes.
\begin{theorem}[Cascaded Parameter Efficiency]
\label{th:linear_stacking_efficiency}
Let $\cC$ be a cascade of $K$ linear WFAs with state
dimensions $d_1, \dots, d_K$. The number of parameters required to
specify the transitions of the cascade is
$O(|\Sigma| \sum_{k=1}^K d_k^2)$. In contrast, a generic
(unconstrained) single WFA with the equivalent state dimension
$D = \sum_{k=1}^K d_k$ requires $O(|\Sigma| (\sum_{k=1}^K d_k)^2)$
parameters. Thus, the stacked architecture enforces a sparsity
constraint that reduces the sample complexity of learning by a factor
of $O(K)$ when $d_k \approx d$.
\end{theorem}

\paragraph{Non-Linear Stacking and Deep Recurrence.}
\label{subsec:nonlinear_stacking}

A natural question arises regarding the interaction between the
rational features and the non-linear components of the Transformer.
While our proposed architecture (Figure~\ref{fig:rt_arch}) uses a
"sidecar" design where the rational states are strictly input-driven,
one could alternatively construct a \emph{Non-Linear Stacked
  Transductor}. In this variant, the input to the $k$-th rational
head is not the original token embedding, but the \emph{non-linear
  output} of the $(k-1)$-th Transformer block.

\begin{figure}[t]
    \centering
        \begin{tikzpicture}[
        font=\sffamily\footnotesize,
        block/.style={draw=blue!50!black, fill=blue!5, rounded corners, minimum width=2.5cm, minimum height=0.8cm, align=center},
        wfa/.style={draw=orange!60!black, fill=orange!10, circle, minimum size=0.8cm, inner sep=1pt, align=center},
        arrow/.style={->, >=Stealth, thick, color=gray!80!black},
        inj/.style={->, >=Stealth, thick, color=green!60!black}
    ]

    % --- PANEL A: RATIONAL TRANSDUCTOR (WIDE) ---
    \node (label_a) at (0, 4.9) {\textbf{(a) Rational Transductor (Wide)}};
    
    \node (in_a) at (0, 0) {Input $x$};
    % The WFA "Sidecar"
    \node[wfa] (wfa_a) at (2.35, 2.45) {WFA\\$h_t$};
    \draw[arrow, dashed, color=orange!80!black] (in_a) -| (wfa_a) node[pos=.6, left, font=\tiny] {Parallel Scan};

    % Transformer Stack
    \node[block] (tf1_a) at (0, 1.2) {TF Block 1};
    \node[block] (tf2_a) at (0, 2.5) {TF Block 2};
    \node[block] (tf3_a) at (0, 3.8) {TF Block 3};

    % Flow
    \draw[arrow] (in_a) -- (tf1_a);
    \draw[arrow] (tf1_a) -- (tf2_a);
    \draw[arrow] (tf2_a) -- (tf3_a);

    % Injections
    \draw[inj] (wfa_a) |- (tf1_a);
    \draw[inj] (wfa_a) -- (tf2_a);
    \draw[inj] (wfa_a) |- (tf3_a);

    % --- PANEL B: NON-LINEAR STACK (DEEP) ---
    \node (label_b) at (5.5, 4.9) {\textbf{(b) Non-Linear Stack (Deep)}};
    
    \node (in_b) at (5.5, 0) {Input $x$};
    
    % Layer 1
    \node[block] (tf1_b) at (5.5, 1.0) {TF Block 1};
    % Recurrence 2 (Interleaved)
    \node[wfa] (wfa2_b) at (5.5, 2.0) {WFA\\$h^{(2)}$};
    % Layer 2
    \node[block] (tf2_b) at (5.5, 3.0) {TF Block 2};
    % Recurrence 3
    \node[wfa] (wfa3_b) at (5.5, 4.2) {WFA\\$h^{(3)}$};

    % Flow (Sequential)
    \draw[arrow] (in_b) -- (tf1_b);
    \draw[arrow] (tf1_b) -- (wfa2_b);
    \draw[arrow] (wfa2_b) -- (tf2_b);
    \draw[arrow] (tf2_b) -- (wfa3_b);

    % Annotations
    \node[align=center, color=green!60!black, font=\tiny] at (1.75, 1.5) {Input-Driven\\State};
    \node[align=center, color=red!60!black, font=\tiny] at (6.8, 2.0) {State depends\\on Layer 1};

    \end{tikzpicture}
    
    \caption{Architectural Comparison.  (a) Wide Recurrence: The
      Rational Transductor computes a single high-dimensional state
      $h_t$ directly from the input via a parallel scan, injecting it
      into all layers.  (b) Deep Recurrence: Stacked architectures
      (e.g., H3, Mamba) interleave recurrence, where Layer $k$ depends
      on the output of Layer $k-1$, reintroducing a sequential
      bottleneck during training.}
    \label{fig:wide_vs_deep}
\end{figure}
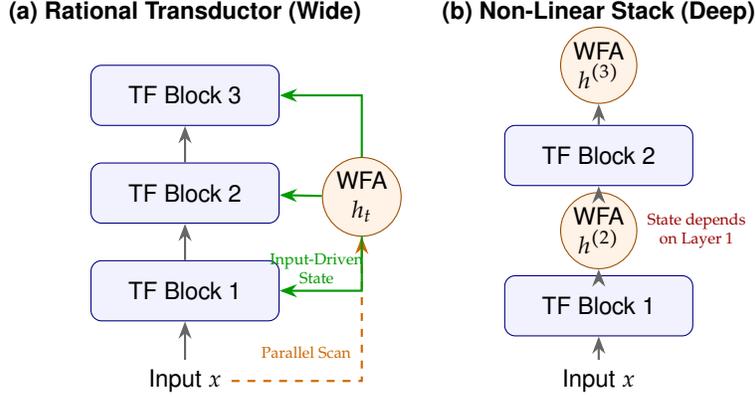

Unlike the linear case, non-linear stacking strictly increases the
expressive capacity of the architecture per unit of state dimension.
Because the transition dynamics of layer $k$ depend on the non-linear
transformation of layer $k-1$, the system can realize functions that
are not realizable by any single linear WFA of the aggregate
dimension.
\begin{theorem}[Non-Linear Irreducibility]
\label{th:nonlinear_irreducibility}
There exist cascades of rational transductors with inter-layer
non-linearities such that any single linear WFA capable of realizing
the same transduction requires a state dimension
$D \gg \sum_{k=1}^K d_k$.
\end{theorem}
\begin{proof}
  Consider a cascade of two WFAs ($K=2$) with state dimensions
  $d_1, d_2 \ge 3$. Let the first WFA compute a state
  $h_t^{(1)} \in \Rset^{d_1}$ and the second compute
  $h_t^{(2)} \in \Rset^{d_2}$, where the input to the second WFA is
  modulated by the state of the first via a multiplicative interaction
  (e.g., attention or gating
  $u_t^{(2)} = u_t^{(1)} \otimes h_t^{(1)}$). To simulate this
  non-linear interaction with a single linear system, one must
  linearize the product state space, requiring a state vector
  isomorphic to the tensor product $h_t^{(1)} \otimes h_t^{(2)}$. The
  dimension of this linearized system is $D = d_1 \times d_2$. For
  all $d_1, d_2 \ge 2$ (and strictly for $d > 2$), the tensor product
  dimension exceeds the sum of dimensions ($d_1 d_2 > d_1 + d_2$).
  Thus, the non-linear cascade represents a function class that is
  more compact (linear vs quadratic in $d$) than any equivalent single
  linear WFA.
\end{proof}

However, this expressivity comes at a steep computational cost. In our
"Single Head" design, the states $\bh_t$ are rational functions of the
\emph{input} $x$.  Consequently, the entire state history for all
layers can be pre-computed in parallel using a single pass of
associative scans ($O(\log T)$ depth).  In contrast, in the Non-Linear
Stack, the input to the $k$-th recurrence $\bu_t^{(k-1)}$ depends on
the attention output of the previous layer.  This introduces a
\emph{layer-wise sequential dependency}: the parallel scan for layer
$k$ cannot commence until layer $k-1$ is fully computed.  For a model
with $L$ layers, the parallel complexity scales as
$O(L \log T)$.\footnote{While theoretical circuit constructions can
  potentially parallelize deep non-linear networks to
  $O(\text{polylog } T)$ depth (e.g., via divide-and-conquer on the
  composition of layers), standard layer-wise implementations in deep
  learning frameworks enforce a sequential dependency of depth $L$,
  making $O(L \log T)$ the practical latency lower bound.} For deep
foundation models ($L \approx 96$) \citep{brown2020language}, this
serialization reintroduces a significant training bottleneck.  Our
architecture thus prioritizes parallel efficiency, opting to capture
complex dependencies via a wider, input-driven rational state rather
than a deep, serial one.

\section{Motivation: Why Rational Features?}
\label{sec:motivation}

To understand why augmenting transformers with weighted automata is
effective, we must consider the complementary strengths of the two
architectures. We argue that Rational Features provide the
\textit{sequential rigidity} that attention mechanisms lack, without
sacrificing the \textit{semantic flexibility} or \textit{computational
  parallelizability} that makes transformers successful.

\textbf{Complementarity of State Tracking and Recall.}  To avoid
ambiguity, we clarify our terminology: we use `Syntax' to refer to
rigid, rule-based state tracking (e.g., maintaining variable values,
counting modulo $k$, or balancing nested brackets), and `Semantics' to
refer to soft, associative recall (e.g., retrieving related concepts
across long distances).  Standard Transformers excel at modeling
\textit{semantic} dependencies via the attention mechanism, which
allows tokens to interact directly regardless of distance. However,
attention is inherently set-based; it treats the context as a ``bag of
tokens'' enriched only by weak positional heuristics. This makes
precise tasks---such as counting modulo $k$, tracking nested
parenthesis depth, or maintaining the status of a variable in code
generation---surprisingly difficult for standard Transformers. In
contrast, Weighted Finite Automata are inherently designed for this
type of \textit{syntactic state tracking}. By computing a rational
feature vector $\bh_t$, the model maintains a compact summary of the
past that is mathematically guaranteed to capture regular languages
exactly. The Rational Transductor thus operates with a
\emph{bicameral} design: we use this term metaphorically to reflect a
division of labor akin to biological cognitive systems, where the WFA
head handles sequential logic and \emph{state tracking} (maintaining
context over time), while the Attention head handles long-range
correlations and \emph{associative recall} (retrieving content from
history).

\textbf{Inductive Bias for Length Generalization.}
The \emph{failure to generalize} in Transformers is largely a failure
of inductive bias. Positional encodings (whether absolute or relative)
are often brittle: a Transformer trained on sequences of length $N$
frequently fails on length $2N$ because the positional representations
drift out of distribution. Rational features solve this by injecting a
\textit{recurrent inductive bias}. The state update
$\bh_t = \sfM \bh_{t-1}$ is time-invariant; the logic used to
transition from step 10 to 11 is identical to the logic used for step
1000 to 1001. If the WFA learns to count or track state correctly on
short sequences, it mathematically must continue to do so for
sequences of arbitrary length. This provides the Transformer with a
stable \emph{anchor} of state information that does not degrade as
sequence length increases.

\textbf{Efficiency via Linear Recurrence.}
Historically, adding recurrence to Transformers (as in Transformer-XL
or hybrids) introduced a sequential bottleneck, preventing parallel
training. Rational features circumvent this trade-off. Because the
transition function is linear ($\bh_t = \sfM_t \bh_{t-1}$), we avoid
the need for sequential backpropagation through time. Instead, the
sequence of states can be computed using \emph{parallel associative
  scans} (or prefix sums) in $O(\log T)$ time on GPUs/TPUs. This
ensures that adding rational features incurs negligible wall-clock
overhead compared to the quadratic cost of attention, maintaining the
scalability required for foundation models.

\section{Expressivity and Complexity}
\label{sec:expressivity}

In this section, we analyze the expressivity and generalization
capabilities of Rational Transductors. We summarize our main findings
in Table~\ref{tab:expressivity_comparison}, which highlights that
Transductors strictly generalize standard Transformers to capture all
regular languages and $\NC^1$ complexity.

\paragraph{Theoretical Setup.}
Unless otherwise stated, all expressivity results assume exact or
bounded-precision arithmetic with error tolerance independent of
sequence length.  For lower bounds, we define ``Standard
Transformers'' as \emph{uniform families of fixed-depth architectures
  without auxiliary memory} (no Chain-of-Thought).  Under these
constraints, Transformers are limited to $\AC^0$ (hard attention)
\citep{hao2022formal} or $\TC^0$ (soft attention)
\citep{merrill2023logic, chiang2024transformers}, preventing them from solving
sequential problems like Parity which require depth scaling with input
length.  Note that while this section focuses on representational
capacity, we explicitly address numerical stability and optimization
dynamics in Section~\ref{sec:learning_theory}.

\begin{table}[t]
\centering
\caption{Theoretical comparison of capabilities between finite-depth
  Transformers and Rational Transductors.
  Transductors strictly expand expressivity to include all regular languages
  while sharing the fundamental limitation on context-free grammars.}
\label{tab:expressivity_comparison}
\vskip .15in
\resizebox{\textwidth}{!}{
\begin{tabular}{@{\hspace{0cm}}l@{\hspace{-.25cm}}cc@{\hspace{.2cm}}l@{\hspace{0cm}}}
\toprule
  \textbf{Task / Property} & \textbf{Transformer} & \textbf{Transductor}
  & \textbf{Theoretical Reason} \\
\midrule
  Parity ($L_{\text{parity}}$) & \hspace{.2cm}\xmark$^{\dagger}$ & \cmark & Limited to $\AC^0$/$\TC^0$ \\
  Modular Counting & \xmark & \cmark & Lack of cyclic state vs. WFA exactness \\
  All Regular Languages & \xmark & \cmark & Star-free limitation vs. Rational completeness \\
  Length Generalization & \xmark & \cmark & Positional drift vs. Time-invariant recurrence \\
  All Context-Free Languages & \xmark & \xmark & Lack of unbounded memory stack \\
\bottomrule
\end{tabular}
}
\vskip .05in
\begin{minipage}{\linewidth}
\footnotesize
$^{\dagger}$ While soft-attention Transformers ($\TC^0$) can
theoretically approximate Parity for fixed lengths via averaging, they
cannot represent the solution \emph{uniformly} for unbounded lengths
without precision scaling or Chain-of-Thought, and empirically fail to
generalize.
\end{minipage}
\end{table}

\paragraph{Interpretation of Results.}
Table~\ref{tab:expressivity_comparison} illustrates that Transductors
occupy just the \emph{right} zone of expressivity. By integrating the
WFA head, they overcome the fundamental inability of standard
Transformers (restricted to $\AC^0$ under hard attention, or $\TC^0$
under soft attention) to \emph{uniformly} handle periodic and
sequential logic (such as Parity and $MOD_k$), effectively solving the
\emph{Regular Gap}. Crucially, however, Transductors do not attempt to
solve the \emph{Context-Free Gap}; like standard Transformers, they
lack the unbounded stack required for tasks like hierarchical bracket
matching. This deliberate design choice preserves the $O(L + \log T)$
parallel time complexity that general RNNs or stack-augmented models
sacrifice.

\paragraph{From Representation to Induction.}
The theorems above (Parity Separation, Modular Counting, Krohn-Rhodes
Completeness) are statements about \emph{representational
  capacity}---they guarantee the \emph{existence} of parameters that
exactly simulate any DFA.  The critical bridge to
\emph{learnability} is provided by the optimization guarantees of
Section~\ref{sec:learning_theory}.  Specifically,
Theorem~\ref{th:gradient_stability} (Gradient Norm Preservation)
establishes that the Cayley parameterization prevents
vanishing/exploding gradients, while
Theorem~\ref{th:length_gen} (Time-Invariant Error
Bounding) shows that the contractive regime ensures approximation
error does not compound with sequence length.  Together, these results
confirm that the representational separation is not merely existential
but \emph{reachable by gradient descent}.

\subsection{Positional Encodings and Generalization}

We first show that our framework strictly generalizes modern
positional encoding schemes.
\begin{lemma}[Positional Encodings are Rational]
\label{lemma:rope}
Over the field of real numbers $\Rset$, let
$\sfP \in \Rset^{T \times d}$ be the matrix of Rotary Positional
Embeddings (RoPE) or sinusoidal encodings. There exists a Weighted
Finite Automaton $\sA = (\Sigma, d, \balpha, \{\sfM_\sigma\})$ such
that for any position $t$, the rational feature vector $\bh_t$ is
exactly the positional encoding at index $t$, independent of the input
tokens.
\end{lemma}

\begin{proof}
  Standard sinusoidal and rotary embeddings are defined by frequencies
  $\Theta = \{\theta_i\}_{i=1}^{d/2}$. The encoding at position $t$
  is typically constructed by rotating pairs of dimensions. Consider
  the block-diagonal rotation matrix $\sfR$:
\begin{equation}
    \sfR = \begin{pmatrix}
    \sfR_{\theta_1} & 0 & \dots & 0 \\
    0 & \sfR_{\theta_2} & \dots & 0 \\
    \vdots & \vdots & \ddots & \vdots \\
    0 & 0 & \dots & \sfR_{\theta_{d/2}}
    \end{pmatrix}, \quad \text{where } \sfR_{\theta_i}
    = \begin{pmatrix}
      \cos \theta_i & -\sin \theta_i \\
      \sin \theta_i & \cos \theta_i
    \end{pmatrix}
\end{equation}
We define a WFA $\sA$ as follows:
\begin{itemize}
\item Let the initial state
  $\balpha = (1, 0, 1, 0, \dots, 1, 0)^\top \in \Rset^d$.
\item Let the transition matrix $\sfM_\sigma = \sfR$ for all
  $\sigma \in \Sigma$. (The transition is input-independent).
\end{itemize}
The rational feature vector at time $t$ evolves as
$\bh_t = \sfM_{x_t} \dots \sfM_{x_1} \balpha = \sfR^t \balpha$. Since
$\sfR$ is block-diagonal, we can analyze the evolution of the $i$-th
pair of dimensions independently:
\begin{equation}
  \begin{pmatrix}
    h_{t, 2i-1} \\
    h_{t, 2i}
  \end{pmatrix}
  = \begin{pmatrix}
    \cos \theta_i & -\sin \theta_i \\
    \sin \theta_i & \cos \theta_i
  \end{pmatrix}^t
  \begin{pmatrix} 1
    \\ 0 \end{pmatrix} =
  \begin{pmatrix}
    \cos (t\theta_i)\\
    \sin (t\theta_i)
  \end{pmatrix}.
\end{equation}
This matches exactly the definition of sinusoidal positional
encodings.  Thus, standard positional information is a special case of
rational features where the transition matrices are unitary and
input-independent. While RoPE is typically applied as a modulation
within the attention mechanism, this lemma demonstrates that the
underlying state tracking mechanism is representable by a WFA.
\end{proof}

This lemma implies that Transductors start with the full capability of
standard Transformers (via PE) but extend it by allowing the
transition matrices $\sfM_\sigma$ to be \textit{input-dependent},
enabling the tracking of semantic states rather than just wall-clock
time.

\begin{figure}[t]
    \centering
    % Define common styles for consistency
    \begin{tikzpicture}[
        >=stealth, 
        auto, 
        thick,
        state_node/.style={circle, draw=blue!60!black, fill=blue!5, thick,
          minimum size=1.0cm}
    ]
    \end{tikzpicture}
    \begin{minipage}{0.48\textwidth}
        \centering
                \begin{tikzpicture}[
            >=stealth, 
            node distance=2.5cm, 
            on grid, 
            auto, 
            thick,
            state/.style={circle, draw=blue!60!black, fill=blue!5, thick, minimum size=1.2cm}
        ]
            % Nodes
            \node[state, initial, initial text=$h_0$] (even) {Even};
            \node[state] (odd) [right=of even] {Odd};

            % Transitions
            \path[->] 
                (even) edge [loop above] node {0} (even)
                (even) edge [bend left] node {1} (odd)
                (odd) edge [loop above] node {0} (odd)
                (odd) edge [bend left] node {1} (even);
            
            % Label
            \node[below=1.5cm of even, xshift=1.25cm, font=\small] {\textbf{(a) Parity WFA}};
        \end{tikzpicture}
    \end{minipage}
    %\hfill
    \begin{minipage}{0.48\textwidth}
        \centering
                \begin{tikzpicture}[
            >=stealth, 
            auto, 
            thick,
            state/.style={circle, draw=orange!80!black, fill=orange!10, thick, minimum size=1.0cm}
        ]
            % Nodes in triangle
            \node[state, initial, initial text=$h_0$] (q0) at (90:1.5cm) {0};
            \node[state] (q1) at (330:1.5cm) {1};
            \node[state] (q2) at (210:1.5cm) {2};
            % Transitions
            \path[->] 
                (q0) edge [loop above] node {0} (q0)
                (q1) edge [loop right] node {0} (q1)
                (q2) edge [loop left] node {0} (q2)
                (q0) edge [bend left=15] node {1} (q1)
                (q1) edge [bend left=15] node {1} (q2)
                (q2) edge [bend left=15] node {1} (q0);
            % Label
            \node[below=2.5cm of q0, font=\small] {\textbf{(b) Modulo-3 Counter}};
        \end{tikzpicture}
    \end{minipage}

    \caption{State tracking mechanisms for exact regular languages.
    \textbf{(a)} The Parity WFA uses a 2-state flip mechanism to
      track $L_{\text{parity}}$.
    \textbf{(b)} The Modulo-3 WFA
      generalizes this to a cyclic group structure to solve $L_k$ for
      $k=3$.
    Input `0` acts as the Identity $\sfI$ (self-loop), while
      input `1` acts as a permutation.}
  \label{fig:automata_examples}
\end{figure}
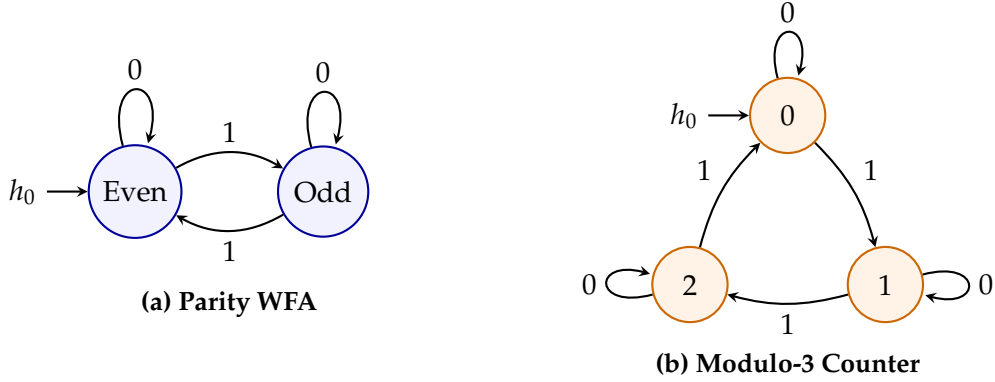

The construction in Lemma 4 formalizes the ``RoPE Trick'' used in
recent linear attention and SSM architectures \citep{dao2025mamba3,
  yang2025path}.  However, while RoPE has traditionally been viewed as
a positional modulation, this result provides the first rigorous proof
that it constitutes a specific \emph{Rational Inductive Bias}.  This
formalization explains precisely why RoPE fails to generalize length
on algorithmic tasks: it constrains the recurrence to track
\emph{input-independent} relative position, rather than the
\emph{input-dependent} state required for semantic generalization.

\subsection{Expressive Separations: Parity and Counting}

\begin{theorem}[The Parity Gap]
\label{th:parity}
Let $L_{\text{parity}}$ be the language of binary strings
$x \in \{0, 1\}^*$ containing an odd number of $1$s.
\begin{enumerate}

\item A standard Transformer with fixed depth, hard attention, and
  \emph{bounded-precision arithmetic} cannot uniformly recognize
  $L_{\text{parity}}$ for input sequences of unbounded length.

\item There exists a Rational Transductor with state dimension $d=2$
  that recognizes $L_{\text{parity}}$ with $100\%$ accuracy for any
  length $T$.

\end{enumerate}
\end{theorem}

\begin{proof}
  Part 1 (Transformer Limitation): Fixed-depth hard-attention
  Transformers compute functions in uniform $\AC^0$
  \citep{hao2022formal}, and $\AC^0$ circuits cannot compute Parity
  \citep{furst1984parity}; see also \citet{hahn2020theoretical}.  While
  soft-attention Transformers are
  strictly more expressive, recent analysis proves they remain bounded
  within $\TC^0$ (Threshold Circuits) even with arbitrary precision,
  provided they operate in fixed depth
  \citep{merrill2023logic, chiang2024transformers}.  Turing completeness is only
  achievable when the model is permitted to generate intermediate
  \emph{Chain-of-Thought} tokens, effectively using the context as a
  read-write tape \citep{perez2021attention,merrill2023expressive}.
  Finite-depth models without such scratchpads remain bounded in
  computational power and, crucially, lack the inductive bias to learn
  robust algorithmic solutions \citep{merrill2022saturated}.

  Part 2 (Rational Feature Solution): We construct a WFA $\sA$ with
  $d=2$ states that tracks the parity of the number of $1$s (see
  Figure~\ref{fig:automata_examples}(a)). Let the state vector be
  $\bh_t \in \Rset^2$, where $\bh_t = (1, 0)^\top$ represents an
  ``Even'' state and $\bh_t = (0, 1)^\top$ represents an ``Odd''
  state.
\begin{itemize}
\item Initial State: $\balpha = (1, 0)^\top$ (0 is an even number).

\item Transitions: For input token `0', the count does not change. We
  set $\sfM_0 = I = \begin{pmatrix} 1 & 0 \\ 0 & 1 \end{pmatrix}$.
  For input token `1', the parity flips. We set
  $\sfM_1 = \begin{pmatrix} 0 & 1 \\ 1 & 0 \end{pmatrix}$.  (Note:
  While a standard Cayley parameterization in $d=2$ is restricted to
  $SO(2)$ and cannot represent the reflection matrix $M_1$, this can
  be realized in the Rational Transductor by augmenting the state
  space to $d \ge 3$. As detailed in Section 6.2, a rotation in three
  dimensions can embed the required reflection logic, ensuring
  algebraic completeness.)
\end{itemize}
The state update $\bh_t = \sfM_{x_t} \bh_{t-1}$ performs exact modular
arithmetic. If the number of ones is even, $\bh_T = (1, 0)^\top$; if
odd, $\bh_T = (0, 1)^\top$. The Rational Transductor injects this
$\bh_T$ into the final layer.  A simple linear classifier (readout
head) $\bw = (0, 1)^\top$ can then perfectly classify the
sequence as $y = \text{sign}(\bw^\top \bh_T)$. This holds for
any sequence length $T$, proving the claim.
\end{proof}

Note that the construction in Part 2 explicitly relies on the learned
parameterization (Regime II), where we can set $\sfM_\sigma$ to
specific permutation matrices.

\begin{theorem}[Exact Modular Counting]
\label{th:mod_counting}
For any fixed integer $k \ge 2$, there exists a Rational Transductor
with state dimension $d = k$ that exactly recognizes the language
\[
L_k = \{ x \in \{0,1\}^* : \#_1(x) \equiv 0 \pmod{k} \}
\]
for sequences of arbitrary length, whereas finite-depth hard-attention
Transformers cannot uniformly recognize $L_k$.
\end{theorem}

\begin{proof}
  The negative result for standard Transformers follows from the fact
  that $MOD_k$ gates are not realizable in $\AC^0$ for any
  $k \ge 2$ \citep{smolensky1987algebraic}. For the Rational
  Transductor construction, we generalize the parity mechanism to the
  cyclic group $\Zset_k$. We define a WFA with dimension $d=k$ where
  the basis vector $\be_i$ represents the current count being
  $i \pmod k$ (see Figure~\ref{fig:automata_examples}(b)).
\begin{itemize}

\item Initial State: $\balpha = \be_0 = (1, 0, \dots, 0)^\top$.

\item Transitions: Let $\sfM_0 = I_k$ (identity). Let $\sfM_1$ be
  the cyclic permutation matrix where $(\sfM_1)_{ij} = 1$ if
  $i \equiv (j+1) \pmod k$ and 0 otherwise.

\end{itemize}
The state update $\bh_t = \sfM_{x_t} \bh_{t-1}$ ensures that if the
current count is $j$, reading a `1' moves the probability mass to
$(j+1) \pmod k$. Thus, $\bh_T = \be_r$ where
$r = \#_1(x) \pmod k$. A linear readout $\bw = \be_0$
correctly identifies if $x \in L_k$.
\end{proof}

\emph{Note:} As with Parity, this capability assumes the learned
regime, where the model can discover and maintain orthogonal
transition matrices.

\begin{theorem}[Exact Arithmetic Evaluation]
\label{th:arithmetic_eval}
Let $val_b\colon \Sigma^* \to \Nset$ be the function that interprets a
string $x \in \{0, \dots, b-1\}^*$ as a number in base $b$ (e.g., for
binary $b=2$, "101" $\mapsto 5$).  There exists a Rational Transductor
with state dimension $d=2$ that computes $val_b(x)$ exactly for
sequences of arbitrary length (assuming computation over $\Rset$).
\end{theorem}

\begin{proof}
  This function corresponds to the Horner scheme for polynomial
  evaluation.  We construct a WFA with state vector
  $\bh_t = [v_t, 1]^\top$, where $v_t$ is the current accumulated
  value.
  \begin{itemize}

  \item \textbf{Initial State:} $\balpha = [0, 1]^\top$.

  \item \textbf{Transitions:} For a digit
    $\sigma \in \{0, \dots, b-1\}$, the update rule is
    $v_t = b \cdot v_{t-1} + \sigma$. This is an affine transformation
    representable by the matrix:
      \begin{equation}
          \sfM_\sigma = \begin{pmatrix} b & \sigma \\ 0 & 1 \end{pmatrix}.
      \end{equation}
  \end{itemize}
  The recurrence yields $\bh_t = \sfM_{x_t} \bh_{t-1}$. The first
  component of $\bh_T$ exactly holds $\sum_{t=1}^T x_t \cdot b^{T-t}$.
  A linear readout $\bw = [1, 0]^\top$ retrieves the integer
  value.  Standard Transformers with bounded activations or attention
  scores cannot represent this unbounded growth function exactly, nor
  can standard RNNs with saturating non-linearities (e.g., $\tanh$).
\end{proof}

\paragraph{Generalization.}
This result illustrates a capability distinct from the finite-state
logic of Parity or Modulo-$k$.  As shown in prior work on the
expressivity of weighted automata \citep{CortesMohri2000}, WFAs can
represent broad families of functions mapping strings to numerical
values, including polynomial evaluations and probabilistic
distributions.  The Rational Transductor inherits this capacity to
model quantitative, unbounded dependencies that are inaccessible to
architectures constrained by saturation or bounded precision.

\paragraph{Recognition of Context-Free Languages.}
While we emphasized the distinction between Regular and Context-Free
capabilities in Table~\ref{tab:expressivity_comparison}, the ability
to perform exact arithmetic (Theorem~\ref{th:arithmetic_eval})
theoretically bridges this gap. As established in
\citep{CortesMohri2000}, the support of rational power series over
infinite fields can characterize certain Context-Free Languages,
including palindromes and Dyck paths, by leveraging arithmetic
"counting" mechanisms to enforce structural constraints (e.g., mapping
balanced nested structures to a zero-sum weight).  Thus, Rational
Transductors possess the latent capacity to "recognize" these
languages through quantitative embedding, though robust learning of
such unbounded arithmetic solutions in practice requires high
precision and specific inductive biases.

\subsection{The Expressive Hierarchy}

We can situate Rational Transductors within the broader hierarchy of
sequence modeling architectures. The Rational Transductor framework
occupies a precise theoretical niche, strictly generalizing both
Rational Power Series and standard Transformers.
\begin{proposition}[Expressive Hierarchy]
\label{prop:expressive_hierarchy}
  Let $\sF_{\text{Rat}}$, $\sF_{\text{TF}}$, $\sF_{\text{RT}}$, and
  $\sF_{\text{RNN}}$ denote the classes of functions computable by
  Rational Power Series (linear WFAs), finite-depth Transformers,
  Rational Transductors, and general Recurrent Neural Networks,
  respectively. The following strict inclusions hold:
\begin{equation}
  (\sF_{\text{Rat}} \cup \sF_{\text{TF}})
  \subsetneq \sF_{\text{RT}} \subsetneq \sF_{\text{RNN}} .
\end{equation}
\end{proposition}

\begin{proof}
We establish each inclusion and its strictness separately.

1. Rational Series ($\sF_{\text{Rat}} \subsetneq \sF_{\text{RT}}$).
Any rational series is computable by a WFA
\citep{Schutzenberger1961}. A Rational Transductor recovers this
exactly by learning the corresponding transitions $\{\sfM_\sigma\}$
and setting the Transformer layers to perform a fixed linear
readout. The inclusion is strict because Transductors can apply
non-linear operations (e.g., softmax attention, layer normalization)
to the sequence of rational states, whereas $\sF_{\text{Rat}}$ is
limited to linear transductions.

2. Transformers ($\sF_{\text{TF}} \subsetneq \sF_{\text{RT}}$).  A
Rational Transductor recovers a standard Transformer if the rational
projection weights are set to zero ($\sfW_{\text{proj}} = 0$). The
strictness is proven by Theorem~\ref{th:parity} and
Theorem~\ref{th:mod_counting}: Transductors can uniformly recognize
Parity and Modular Counting languages, whereas standard Transformers
(limited to $\AC^0$) cannot.

3. Recurrent Neural Networks
($\sF_{\text{RT}} \subsetneq \sF_{\text{RNN}}$).  Rational
Transductors are a specific instance of RNNs where the state
transition is strictly linear ($h_t = \sfM_{x_t} h_{t-1}$).  General
RNNs allow for non-linear state dynamics (e.g., $\tanh$), which
encompass linear updates in the small-signal regime (linear region of
the activation).  The inclusion is strict because the linear state
dynamics of the Transductor admit a parallel-prefix decomposition,
limiting them to $\PNC^1$.  In contrast, general RNNs with non-linear
activations and unbounded precision are $\sfP$-complete
(sequentially strictly harder).  While linear Transductors recognize
some $\NC^1$-complete languages (e.g., the word problem of $S_5$,
Proposition~\ref{prop:word_problem}), they lack the unbounded stack
required to recognize \emph{all} Context-Free Languages (e.g.,
Dyck-$k$ of arbitrary depth) uniformly.
\end{proof}

\textbf{Algebraic Completeness via Krohn-Rhodes.} From the
perspective of algebraic automata theory, the Krohn-Rhodes theorem
establishes that any finite state machine can be decomposed into a
cascade of simple groups (counters) and aperiodic monoids (threshold
logic). Hard-attention Transformers are known to be limited to the
aperiodic component (star-free languages) \citep{yang2024masked}.
By incorporating linear recurrence, which naturally implements cyclic
group operations (as seen in Theorem~\ref{th:mod_counting}),
Transductors effectively recover the group-theoretic component. Thus,
Transductors provide a structurally complete architecture capable of
modeling both the periodic and aperiodic sub-structures of all regular
languages.

\textbf{Descriptive Complexity and MSO.} In the framework of
descriptive complexity, standard Transformers are often associated
with First-Order Logic ($\FO[<]$), which describes star-free
regular languages but fails to capture modulo counting quantifiers.
In contrast, Weighted Finite Automata are intimately linked to Monadic
Second-Order Logic ($\MSO[<]$) over fields. By bridging these
architectures, Transductors effectively lift the expressivity of the
model from the limitations of first-order predicates to the fuller
expressive range of monadic second-order logic on sequences.

\begin{proposition}[Effective Capacity Increase]
  For any fixed horizon $L$, the pseudo-dimension of Rational
  Transductors strictly exceeds that of finite-depth Transformers with
  the same hidden dimension, due to the ability to linearly separate
  histories that are indistinguishable under attention-only
  architectures.
\end{proposition}

\begin{proof}
  Let $\cH_{\text{TF}}$ and $\cH_{\text{RT}}$ be the hypothesis
  classes of finite-depth Transformers and Transductors. Consider the
  set of all binary sequences of length $L$, $S_L = \{0, 1\}^L$, and
  the subset of labelings defined by $y(x) = \text{Parity}(x)$.
  A Rational Transductor with $d=2$ can realizably separate $S_L$
  according to parity labelings (Theorem~\ref{th:parity}). Conversely,
  for fixed model size, there exists a length $L$ such that a standard
  hard-attention Transformer (whose functions lie in $\AC^0$) cannot
  compute Parity \citep{hao2022formal, furst1984parity}, failing to
  shatter $S_L$. Since
  $\cH_{\text{TF}} \subseteq \cH_{\text{RT}}$ and Transductors can
  realize dichotomies (parity) that Transformers cannot,
  $\text{Pdim}(\cH_{\text{RT}}) > \text{Pdim}(\cH_{\text{TF}})$.
\end{proof}

\begin{theorem}[Representational Completeness via Krohn-Rhodes]
\label{th:krohn_rhodes}
Let $\cA$ be any deterministic finite automaton with state set $Q$ and
input alphabet $\Sigma$.  There exists a parameter setting $\theta$
for a \emph{Stacked} Rational Transductor (see
Section~\ref{subsec:architectural_extensions}) of finite depth $L$ and
width $d$ such that the model exactly simulates the state transitions
of $\cA$.  Specifically, the architecture admits a parameterization
that realizes the Krohn-Rhodes decomposition
$(S_L \wr \dots \wr S_1) \wr (R_L \wr \dots \wr R_1)$, where
orthogonal Rational Feature Heads implement the simple groups $S_i$,
stochastic Rational Feature Heads implement the reset (flip-flop)
components $R_i$, and the Transformer layers implement the cascading
feedback functions required by the wreath product.
\end{theorem}

\begin{proof}
  The Krohn-Rhodes theorem states that any finite automaton can be
  decomposed into a cascade (wreath product) of finite simple groups
  and aperiodic monoids (resets) \citep{krohn1965algebraic}.  We
  construct a parameterization of the Rational Transductor that
  physically instantiates this cascade.

  \textbf{1. The Wreath Product Structure.}  A wreath product
  $M_2 \wr M_1$ is driven by a feedback function
  $\phi: Q_1 \times \Sigma \to \text{End}(Q_2)$, where the update rule
  for the second machine depends on the state of the first:
  \begin{equation}
      q_t^{(2)} = q_{t-1}^{(2)} \cdot \phi(q_{t-1}^{(1)}, x_t).
  \end{equation}

  \textbf{2. Structural Mapping.}  We map the cascade layers
  $1 \dots L$ to the Transductor layers.
  \begin{itemize}

  \item \textbf{Group Components ($S_i$):} The Rational Feature Head
    at layer $i$ is parameterized to implement the group
    operations. For a simple group $G$, we set the transition matrices
    $\sfM_\sigma$ to be the permutation matrices representing the
    group elements.

  \item \textbf{Reset Components ($R_i$):} A reset automaton (flip-flop)
    on state set $Q_i$ has transitions that are either the identity or
    a constant map $q \mapsto a$.  A constant map is the rank-one
    matrix $\be_a \mathbf{1}^\top$, which is column-stochastic; hence a
    stochastic head with $d = |Q_i|$ states implements $R_i$ exactly.
    (In particular, the flip-flop monoid $U_2$ is realized with
    $d = 2$.)  We emphasize that this step does \emph{not} rely on the
    attention mechanism: holding the most recent reset symbol requires
    attending to the rightmost occurrence, which unique hard attention
    with leftmost tie-breaking (equivalently, finite-precision softmax
    attention) cannot do \citep{jerad2025unique}, whereas the
    stochastic head carries this state in its recurrence.

  \item \textbf{Feedback Components ($\phi$):} The
    Transformer block between layer $i$ and $i+1$ implements the
    feedback function $\phi$. The input to layer $i+1$'s head is
    $u_t = \text{FFN}(\text{LayerNorm}(z_t^{(i)}))$.
  \end{itemize}

  \textbf{3. Exact Implementation of Feedback Logic.}  The function
  $\phi$ is a map over a finite domain (discrete states $Q_1$ and
  alphabet $\Sigma$).  It is a standard result that a Feed-Forward
  Network (FFN) with ReLU activations and sufficient width can
  \emph{exactly} represent any function over a finite boolean domain
  (e.g., implementing arbitrary logic gates or look-up tables).  Thus,
  there exists a setting of the FFN weights such that:
  \begin{equation}
    \sfM_t^{(i+1)} = \text{Proj}( \text{FFN}( \text{Embed}(q_{t-1}^{(1)})
    + \text{Embed}(x_t) ) )
  \end{equation}
  exactly recovers the transition operator required by the wreath
  product logic $\phi(q_{t-1}^{(1)}, x_t)$.

  \textbf{Conclusion.}  Since the architecture contains components
  capable of representing both the algebraic primitives (groups and
  resets, via orthogonal and stochastic recurrence matrices) and the
  arbitrary Boolean glue logic (via FFNs), there exists a parameter
  configuration that exactly simulates the full Krohn-Rhodes cascade.
  (A stochastic head with $d = |Q|$ can in fact simulate $\cA$
  directly, using its $0/1$ transition matrices; the cascade form is
  of interest for the division of labor it exhibits.)
\end{proof}

\textbf{Algebraic Division of Labor.}
This structural alignment highlights a natural division of labor
within the architecture. The Rational Feature Head is uniquely suited
to implement the \emph{permutation} components (finite simple groups)
of the decomposition via orthogonal recurrence matrices—precisely the
structures that standard self-attention fails to represent
uniformly. The stochastic head implements the \emph{reset}
components (flip-flops), and the Transformer backbone (via FFNs)
implements the \emph{cascading} feedback required to glue the
decomposition together; attention may additionally supply threshold
logic but is not required for the construction.

\textbf{Significance: A Constructive Algebraic Completion.}  This
result offers a fundamental algebraic justification for the Rational
Transductor architecture, moving beyond simple expressive capacity
arguments. It is established that hard-attention Transformers are
limited to star-free languages \citep{yang2024masked}, which
algebraically correspond to aperiodic monoids (threshold logic and
resets), and with leftmost tie-breaking to the even smaller class
of R-trivial languages, which excludes flip-flops
\citep{jerad2025unique}. By augmenting attention with a linear
recurrence capable of implementing finite simple groups
(orthogonal heads) and resets (stochastic heads), the Rational
Transductor effectively physically instantiates the Krohn-Rhodes
decomposition of finite automata. In this view, the "Deep Rational
Injection" mechanism functions as the wreath product operator,
cascading the cyclic state information (from the group component) into
the aperiodic logic (of the attention component). Thus, the
architecture is not merely a heuristic ensemble, but a structurally
complete neural realization of the fundamental algebraic components
required to recognize \emph{any} regular language, explicitly
repairing the specific group-theoretic deficiency of the Transformer.

\begin{theorem}[Logical Characterization via Weighted MSO]
\label{th:logical_completeness}
Let $\text{W-MSO}[<]$ be the Weighted Monadic Second-Order Logic over
the commutative semiring $\Kset$ (in this work, the field
$\Kset = (\Rset, +, \times)$). Under the restriction of \emph{hard
  attention} (which implements $\FO[<]$ selection) and the assumption
that the Transformer layers are constrained to linear projections
(effectively disabling non-linear FFNs), the class of functions
representable by Rational Transductors $\cF_{\text{RT}}$ coincides
exactly with the class of functions definable in
$\text{W-MSO}[<]$. That is:
\begin{equation}
  \cF_{\text{RT}}[\text{HardAttn, Linear}]
  = \llbracket \text{W-MSO}[<] \rrbracket.
\end{equation}
\end{theorem}

\begin{proof}
  We establish the equality by mutual inclusion.

  \textbf{1. Lower Bound
    ($\llbracket \text{W-MSO} \rrbracket \subseteq
    \cF_{\text{RT}}$).}  A fundamental result by
  \citet{droste2007weighted} establishes that for any formula
  $\phi \in \text{W-MSO}[<]$ over a field, there exists a Weighted
  Finite Automaton (WFA) $\cA_\phi$ that computes the same
  series.  Since the Rational Feature Head of a Transductor can
  implement any WFA (by learning the transition matrices
  $\sfM_\sigma$), and the Transformer layers can implement the linear
  readout (by setting $\sfW_{\text{proj}}$ to inject the final state
  and attention to identity), every W-MSO formula is realizable by a
  Rational Transductor.

  \textbf{2. Upper Bound
    ($\cF_{\text{RT}} \subseteq \llbracket \text{W-MSO}
    \rrbracket$).}  We show that the computation of a Rational
  Transductor is definable in W-MSO.
  \begin{itemize}

  \item \textbf{Recurrence:} The linear update
    $\bh_t = \sfM_{x_t} \bh_{t-1}$ is a regular recurrence, which is
    known to be definable in MSO (specifically, the relation "state at
    $t$ is $q$" is MSO-definable).

  \item \textbf{Hard Attention:} Hard attention performs a selection
    $y_t = x_{k}$ where $k = \argmax_j (q_t \cdot k_j)$. The $\argmax$
    and indexing operations are First-Order ($\FO[<]$)
    definable relations.

  \item \textbf{Composition:} Since
    $\FO[<] \subset \MSO[<]$, the composition of the
    recurrence (MSO) and the attention mechanism (FO) remains within
    W-MSO.

  \end{itemize}
  Thus, the entire input-output mapping of the hard-attention Rational
  Transductor can be described by a single Weighted MSO formula.
\end{proof}

\textbf{Soft Attention and Extended Expressivity.}
We note that standard Transformers use \emph{soft attention}, which
strictly exceeds the expressivity of $\FO[<]$ and can
approximate specific Context-Free Languages like Dyck-1 ($a^n b^n$)
\citep{merrill2023logic, chiang2024transformers}. Consequently, a Rational Transductor
with soft attention theoretically exceeds the W-MSO upper
bound. However, the hard-attention equivalence (Theorem
\ref{th:logical_completeness}) serves to formalize the specific
contribution of the Rational Head: it provides the \emph{monadic
  quantifiers} (state tracking) that are structurally absent in the
attention mechanism, regardless of precision.

\paragraph{Remark on Logical Completeness (Weighted MSO).}
This algebraic perspective naturally extends to descriptive
complexity. It is a classical result that Weighted Finite Automata
are expressively equivalent to \emph{Weighted Monadic Second-Order
  Logic} ($\text{W-MSO}[<]$) \citep{droste2007weighted}. Standard
Transformers are known to correspond to First-Order Logic
($\text{FO}[<]$), which effectively captures the "aperiodic" component
of language but fails to express the "group" component (modulo
counting). By structurally integrating a WFA (which captures the full
power of MSO on sequences) with Attention (FO), the Rational
Transductor architecture effectively bridges the gap between
First-Order and Second-Order logic. This implies the model is
\emph{logically complete} for quantitative regular properties,
contrasting with standard Transformers which are strictly limited to
first-order expressivity.

\subsection{Structural Characterization}

We can rigorously characterize the class of functions computable by
Rational Transductors $\sF_{\text{RT}}$ as the exact composition of
the Transformer class with the class of Vector-Valued Rational
Functions.
\begin{definition}[Vector-Valued Rational Functions]
  Let $\Sigma$ be an input alphabet. A function
  $\Phi\colon \Sigma^* \to (\Rset^d)^*$ is a \emph{Vector-Valued Rational
    Function} if it is realizable by a Weighted Finite Automaton
  (WFA). That is, there exists a linear representation
  $(\balpha, \{\sfM_\sigma\}_{\sigma \in \Sigma})$ such that for any
  input $x = (x_1, \dots, x_T)$, the output is the sequence of state
  vectors $\Phi(x) = (\bh_1, \dots, \bh_T)$ defined by the recurrence
  $\bh_t = \sfM_{x_t} \bh_{t-1}$ (with $\bh_0 = \balpha$). We denote
  this class as $\sT_{\text{Rat}}$.
\end{definition}

\begin{theorem}[Decomposition Characterization]
\label{th:decomposition}
Let $\sT_{\text{Rat}}$ be the class of vector-valued rational
functions defined above. Let $\sF_{\text{TF}}$ be the class of
functions computable by finite-depth Transformers. The class of
Rational Transductor functions is exactly the composition of these two
classes:
\begin{equation}
    \sF_{\text{RT}} = \sF_{\text{TF}} \circ \sT_{\text{Rat}}
\end{equation}
That is, a function $F$ is a Rational Transductor if and only if
$F(x) = G(\Phi(x))$ for some $G \in \sF_{\text{TF}}$ and
$\Phi \in \sT_{\text{Rat}}$.
\end{theorem}

\begin{proof}
  Direction 1
  ($\sF_{\text{RT}} \subseteq \sF_{\text{TF}} \circ
  \sT_{\text{Rat}}$): We show that the Deep Rational Injection
  mechanism can be simulated by a standard Transformer given access to
  the rational states. A Rational Transductor computes layers via
  $\bz_t^{(l+1)} = \text{TF}_l\paren*{\bz_t^{(l)} + \sfW_{\text{proj}}^{(l)}
  \bh_t}$.  Consider a standard Transformer $G$ taking the concatenated
  input $u_t = [\bx_t; \bh_t]$.  A standard Transformer can
  simulate the Transductor's deep injection by: (1) Dedicating a
  subspace of its residual stream to copy $\bh_t$ forward to all
  layers (using identity attention/FFN weights); and (2) At each layer
  $l$, applying the linear operation corresponding to
  $\sfW_{\text{proj}}^{(l)}$ to this subspace and adding it to the
  processing stream.  Thus, any function computed by a Transductor is
  computable by a standard Transformer acting on the sequence
  $(x, \Phi(x))$, where $\Phi \in \sT_{\text{Rat}}$.

  Direction 2
  ($\sF_{\text{TF}} \circ \sT_{\text{Rat}} \subseteq
  \sF_{\text{RT}}$): Consider an arbitrary composition
  $F(x) = G(\Phi(x))$, where $G$ is a Transformer and $\Phi$ is a
  rational function.  By definition, $\Phi(x)$ corresponds to the
  state sequence $\bh_t$ of some linear representation. A Rational
  Transductor can implement this composition by: (1) Configuring its
  WFA head to generate $\bh_t$; (2) Setting the first projection
  $\sfW_{\text{proj}}^{(0)}$ to inject $\bh_t$ directly into the input
  embedding (simulating the input to $G$); and (3) Setting subsequent
  projections $\sfW_{\text{proj}}^{(l)} = 0$ for $l > 0$. The
  remaining Transformer layers then implement $G$ exactly.
\end{proof}

\begin{corollary}[Maximality Under Linear Recurrence]
\label{cor:maximality}
Any extension of finite-depth Transformers that augments the model
with a fixed-dimensional, associative, linear state update computable
via parallel prefix operations cannot be strictly more expressive than
Rational Transductors. Any further expressivity gain requires either
non-linear state dynamics or depth growing with input length.
\end{corollary}

\begin{proof}
  An "associative, linear state update" is defined by a recurrence of
  the form $\bh_t = \sfM(x_t) \bh_{t-1}$. This is mathematically
  isomorphic to the state evolution of a Weighted Finite
  Automaton. Since Transductors are defined
  (Theorem~\ref{th:decomposition}) as the composition of Transformers
  with the \textit{entire class} of such rational functions
  $\sT_{\text{Rat}}$, they necessarily subsume any specific
  instance of this recurrence pattern.
\end{proof}

\textbf{Minimal Expressive Extension.}  Transductors constitute a
minimal extension of standard Transformers that suffices to escape the
$\AC^0$ barrier while preserving parallelizability. Thus,
Transductors represent the smallest algebraically closed class of
extensions that enable exact regular-language computation and length
generalization.

\begin{proposition}[Virtual Tensorization via Attention]
\label{prop:virtual_tensorization}
Let $\bh_t \in \Rset^d$ be the rational feature state at time $t$.
The Deep Rational Injection mechanism, combined with a single
Self-Attention head, allows the Transductor to compute decision
boundaries that are linear in the Kronecker product space
$\bh_t \otimes \bh_{t'}$. Consequently, a Rational Transductor with
state dimension $d$ can approximate the expressive power of a
higher-order Weighted Finite Automaton with state dimension $d^2$,
without explicitly materializing the $O(d^6)$ transition tensor.
\end{proposition}

\begin{proof}
  Consider the input to the attention mechanism at layer $l$, denoted
  by $\wt \bz_t$. Due to Deep Rational Injection, this
  vector is the sum of the semantic embedding and the projected
  rational state:
    \begin{equation}
        \wt \bz_t = \bz_t + \sfW_{\text{proj}} \bh_t.
\end{equation}
    The self-attention mechanism computes alignment scores $A_{t,t'}$
    between positions $t$ (query) and $t'$ (key) via the inner product
    of projected representations:
    \begin{align}
      \text{Score}(t, t')
      & = (\sfW_Q \wt \bz_t)^\top (\sfW_K \wt \bz_{t'}) \\
      & = (\bz_t + \sfW_{\text{proj}} \bh_t)^\top \sfW_Q^\top \sfW_K (\bz_{t'} + \sfW_{\text{proj}} \bh_{t'}).
\end{align}
    Expanding this quadratic form yields four terms. We focus on the
    term governing the interaction between the rational states:
    \begin{equation}
      S_{\text{rational}}(t, t')
      = \bh_t^\top \paren*{ \sfW_{\text{proj}}^\top \sfW_Q^\top \sfW_K \sfW_{\text{proj}} } \bh_{t'}.
\end{equation}
    Let
    $\sfM = \sfW_{\text{proj}}^\top \sfW_Q^\top \sfW_K
    \sfW_{\text{proj}}$ be the effective interaction matrix. Using
    the vectorization identity
    $\ba^\top \sfM \bb = \text{vec}(\sfM)^\top
    (\bb \otimes \ba)$, we can rewrite the score as:
    \begin{equation}
        S_{\text{rational}}(t, t') = \text{vec}(\sfM)^\top (\bh_{t'} \otimes \bh_t).
\end{equation}
    This demonstrates that the attention mechanism implicitly computes
    a linear function over the tensor product of the states.
    Physically stacking two WFAs would produce a state space of
    dimension $d_1 + d_2$, but the attention mechanism allows the
    model to leverage the multiplicative interactions of the states,
    effectively simulating a state space of dimension $d \times d$.
    Thus, the Transductor can capture second-order dependencies (e.g.,
    correlating the state at the start of a clause with the state at
    the end) that would otherwise require a significantly larger
    linear automaton to represent.
\end{proof}

\textbf{Significance: Implicit State Expansion.}
The theoretical significance of this result is twofold. First, it
provides a rigorous justification for the "Sidecar" design (wide,
parallel recurrence) over the "Stacked" design (deep, serial
recurrence) used in recent hybrid models. While stacking linear
layers physically increases the state depth, it reintroduces
sequential dependencies that hinder training
\citep{gu2021efficiently}. In contrast, Virtual Tensorization reveals
that the \emph{interaction} between the Rational Head and the
Attention mechanism naturally simulates a higher-order automaton. By
injecting the linear state $h_t$ into the query/key projections of the
attention layer, the dot-product $\text{Attention}(Q, K)$ effectively
computes a kernel over the tensor product space $h_t \otimes h_{t'}$.
This implies that a Rational Transductor with state dimension $d$ can
approximate the decision boundaries of a much larger automaton of
dimension $d^2$ (or higher, with multiple layers), achieving the
expressive benefits of deep recurrence without incurring the
$O(L \log T)$ serialization penalty or the $O(d^6)$ cost of explicitly
simulating higher-order tensor dynamics. Thus, the architecture
achieves a "virtual depth" via the multiplicative interactions of
attention, maintaining the optimal $O(\log T)$ parallel efficiency of
the linear scan while capturing the complex, hierarchical dependencies
usually associated with deep, non-linear stacks.

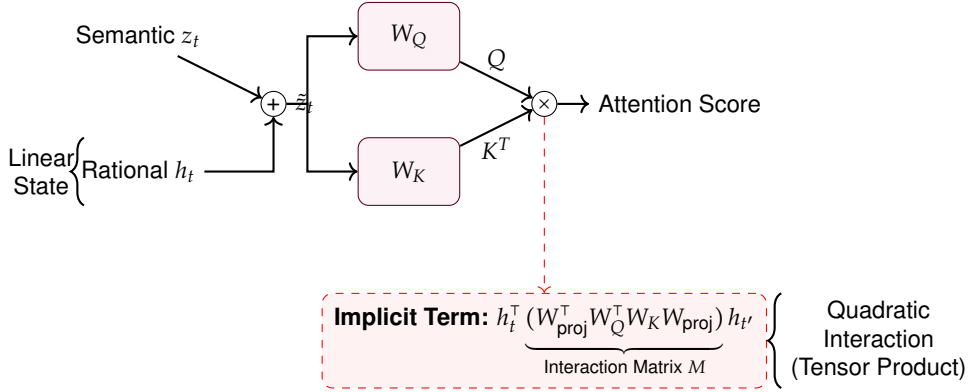
\begin{figure}[t]
  \centering
      \begin{tikzpicture}[
        font=\sffamily\small,
        scale=0.9, transform shape,
        tensor/.style={draw=purple!50!black, fill=purple!5, rectangle, rounded corners, minimum height=1cm, minimum width=1.5cm},
        vec/.style={draw=black, rectangle, fill=white, minimum height=1.5cm, minimum width=0.4cm},
        op/.style={circle, draw=black, inner sep=1pt}
    ]

    % Inputs
    \node (z_t) at (0, 2) {Semantic $z_t$};
    \node (h_t) at (0, 0) {Rational $h_t$};

    % Injection (Sum)
    \node[op] (plus) at (2, 1) {+};
    \draw[->, thick] (z_t) -- (plus);
    \draw[->, thick] (h_t) -| (plus);
    \node[right] at (plus.east) {$\tilde{z}_t$};

    % Projections (Q and K)
    \node[tensor] (WQ) at (4, 2) {$W_Q$};
    \node[tensor] (WK) at (4, 0) {$W_K$};
    
    \draw[->, thick] (plus) -- (2.5, 1) |- (WQ);
    \draw[->, thick] (plus) -- (2.5, 1) |- (WK);

    % Dot Product
    \node[op] (dot) at (6, 1) {$\times$};
    \draw[->, thick] (WQ) -- (dot) node[midway, above] {$Q$};
    \draw[->, thick] (WK) -- (dot) node[midway, below] {$K^T$};

    % Result / Callout
    \node (score) at (8, 1) {Attention Score};
    \draw[->, thick] (dot) -- (score);

    % The "Virtual Tensor" Annotation
    \node[draw=red, dashed, inner sep=5pt, rounded corners, fill=red!5] (expansion) at (6, -2.5) {
        \textbf{Implicit Term:}
        $h_t^\top \underbrace{(W_{\text{proj}}^\top W_Q^\top W_K W_{\text{proj}})}_{\text{Interaction Matrix } M} h_{t'}$
    };
    
    \draw[->, dashed, color=red] (dot) -- (expansion);
    \draw[decorate, decoration={brace, mirror, amplitude=5pt}, thick] (-0.8, 0.5) -- (-0.8, -0.5) node[midway, left, align=center] {Linear\\State};
    \draw[decorate, decoration={brace, mirror, amplitude=5pt}, thick] (9.5, -1.8) -- (9.5, -3.2) node[midway, right, align=center] {Quadratic\\Interaction\\(Tensor Product)};

    \end{tikzpicture}
  \caption{Virtual Tensorization. By injecting the linear rational
    state $h_t$ into the Attention mechanism, the dot product $QK^T$
    implicitly computes quadratic terms of the form
    $h_t^\top M h_{t'}$. This effectively simulates a kernel over the
    tensor product space $h_t \otimes h_{t'}$, enabling the model to
    capture higher-order dependencies without explicitly materializing
    the $O(d^2)$ state space.}
    \label{fig:virtual_tensor}
\end{figure}

\paragraph{Remark: Neural Implementation of Automata Operations.}
The Rational Transductor architecture can be understood as a physical
realization of the closure properties of Rational Power Series:
\begin{itemize}

\item \textbf{Parallel Heads as Direct Sum ($\oplus$):} Instantiating
  multiple independent rational heads (e.g., in the Universal
  Transductor) corresponds to the \textbf{Sum} operation. If head $A$
  tracks parity and head $B$ tracks modulo-3, the concatenated state
  $h = [h_A; h_B]$ represents the direct sum automaton
  $\cA_A \oplus \cA_B$, capable of tracking both
  features in parallel.

\item \textbf{Attention as Tensor Product ($\otimes$):} The
  \textbf{Intersection} of two regular languages (e.g., "strings that
  have odd parity AND length $0 \pmod 3$") requires a state space
  isomorphic to the tensor product of the constituent automata. While
  the rational layer computes the sum, the subsequent non-linear
  mixing (via Attention or MLP) approximates the
  \textbf{Cross-Product} or \textbf{Hadamard Product}, allowing the
  model to form complex decision boundaries based on the conjunction
  of independent rational features.

\item \textbf{Depth as Composition ($\circ$):} Finally, stacking
  blocks corresponds to the \textbf{Composition} of transductions (or
  the cascade product), allowing for hierarchical feature extraction
  where the rational state at layer $l$ depends on the semantic
  features extracted at layer $l-1$.

\end{itemize}
Thus, the architecture provides a complete differentiable substrate
for the algebra of rational series.

\subsection{Circuit Complexity Characterization}

We now characterize the computational power of Transductors from the
perspective of circuit complexity. This framing provides a rigorous
explanation for why Transductors solve the Parity Gap.

\textbf{Transformers: $\AC^0$, $\TC^0$, and Robustness.}
Finite-depth Transformers with hard attention compute Boolean
functions within uniform $\AC^0$ \citep{hao2022formal}; see also
\citet{hahn2020theoretical}. We note
that "standard" Transformers using \emph{soft attention} are strictly
more expressive, falling into the complexity class $\TC^0$ (Threshold
Circuits) \citep{merrill2022saturated}. While $\TC^0$ models can
theoretically express Parity (e.g., via global averaging), they lack
the inductive bias to learn these solutions robustly. Recent work
identifying C-RASP as a formal model of Transformer length
generalization suggests that while tasks like Modulo Counting are in
$\TC^0$, they are not expressible in the C-RASP fragment, explaining
why standard models fail to learn them in a way that generalizes to
unseen lengths. In contrast, Rational Transductors structurally
implement the exact $\PNC^1$ mechanism required.

\begin{theorem}[Circuit Upper Bound for Transductors]
\label{th:rbt-nc1-upper}
For fixed model parameters, any function computed by a Rational
Transductor lies in $\PNC^1$ (and strictly in $\NC^1$
under the assumption of finite precision/fields).
\end{theorem}

\begin{proof}
  The complexity class $\PNC^1$ (Probabilistic $\NC^1$)
  characterizes problems solvable by uniform arithmetic circuits of
  product-depth $O(\log T)$ and polynomial size. The Transductor
  computation proceeds in two stages:
\begin{enumerate}

\item \textbf{Rational State Computation.} The rational feature state
  $\bh_t = \sfM_{x_t} \cdots \sfM_{x_1} \balpha$ involves an iterated
  product of $T$ matrices.  While iterated multiplication of matrices
  over a finite field (or bounded width branching programs) is in
  $\NC^1$ \citep{barrington1989bounded}, the simulation of
  weighted automata over rationals or integers is complete for
  $\PNC^1$ \citep{jung1985parallel}.  Since the state update
  is associative, it admits a parallel-prefix tree decomposition of
  depth $O(\log T)$, placing it within $\PNC^1$.

\item \textbf{Transformer Processing.} Standard Transformer layers
  perform aggregation (attention) over $T$ elements.  Even under the
  stronger assumption of soft attention (which places Transformers in
  $\TC^0$), this class is contained within $\NC^1$ (using the standard
  inclusion $\TC^0 \subseteq \NC^1$) and thus within $\PNC^1$
  \citep{hao2022formal, merrill2023logic}.

\end{enumerate}
Since $\PNC^1$ is closed under composition and contains the
Transformer's complexity class, the total Transductor computation lies
in $\PNC^1$.
\end{proof}

\paragraph{Remark on Tightness.}
This upper bound is essentially tight from a circuit complexity
perspective. The problem of simulating a general WFA (or iterated
integer matrix multiplication) is known to be
$\PNC^1$-complete.  Claiming a tighter bound of
$\NC^1$ for the general case would imply
$\NC^1 = \PNC^1$, solving a major open problem in
complexity theory.  However, if we restrict the model to \emph{bounded
  precision} arithmetic or operations over a finite field (effectively
treating the WFA as a DFA or NFA), the complexity collapses to
$\NC^1$ via Barrington's Theorem.  Rational Transductors thus
occupy the complexity class $\PNC^1$, which sits between
$\NC^1$ and $\sfL$ (Log-Space), strictly separating them
from the $\AC^0$ (or $\TC^0$) limitations of standard
Transformers.

\begin{proposition}[Group Word Problems and $\NC^1$-Completeness]
\label{prop:word_problem}
  Let $G$ be a finite group with identity $e$ and let
  $W_G = \{ g_1 \cdots g_T \in G^* : g_T \cdots g_1 = e \}$ be its word
  problem.  A rational head with $d = |G|$ states, initial vector
  $\balpha = \be_e$, and transition matrices $\sfM_g$ given by the
  permutation matrices of the regular representation
  ($\sfM_g \be_h = \be_{gh}$) recognizes $W_G$ exactly for every length
  $T$.  For $G = S_5$, the language $W_G$ is $\NC^1$-complete under
  $\AC^0$ reductions \citep{barrington1989bounded}.  Consequently,
  Rational Transductors recognize $\NC^1$-complete languages, and
  $\sF_{\mathrm{RT}} \subseteq \TC^0$ would imply $\TC^0 = \NC^1$.
\end{proposition}

\begin{proof}
  By induction on $t$, $\bh_t = \be_{g_t \cdots g_1}$, so
  $x \in W_G$ if and only if $\be_e^\top \bh_T = 1$; the matrices
  $\sfM_g$ are permutation matrices, hence orthogonal, and the
  construction is uniform in $T$.  Barrington's theorem states that
  every language in $\NC^1$ reduces to $W_{S_5}$ by a uniform $\AC^0$
  (indeed projection) reduction; since $\TC^0$ is closed under such
  reductions, $W_{S_5} \in \TC^0$ would give $\NC^1 \subseteq \TC^0$.
  We stress the scope of this statement: the rational head reads its
  input once, in order, i.e., it is an \emph{oblivious, read-once}
  linear branching program of width $d$.  Over a finite field, such a
  head recognizes exactly the regular languages; the
  $\NC^1$-hardness of $W_{S_5}$ resides in the reduction, not in the
  head, and the proposition does not assert that $\NC^1 \subseteq
  \sF_{\mathrm{RT}}$.
\end{proof}

\begin{theorem}[Circuit Lower Bound for Transductors]
\label{th:rbt-nc1-lower}
There exist functions computable by Rational Transductors that are not
in $\AC^0$.
\end{theorem}

\begin{proof}
  By Theorem~\ref{th:parity} and Theorem~\ref{th:mod_counting},
  Transductors can uniformly compute Parity and Modular Counting over
  unbounded input lengths. Since $MOD_k$ is not computable in
  $\AC^0$ for any $k \ge 2$ \citep{smolensky1987algebraic},
  this establishes that
  $\sF_{\mathrm{RT}} \not\subseteq \AC^0$.
\end{proof}

\begin{corollary}[Circuit Complexity Sandwich]
\label{cor:rbt-sandwich}
Let $\sF_{\mathrm{TF}}$ denote the class of functions computable by
fixed-depth hard-attention Transformers.  Then
\[
    \sF_{\mathrm{TF}} \;\subsetneq\; \sF_{\mathrm{RT}} \;\subseteq\; \PNC^1 ,
    \qquad
    \sF_{\mathrm{TF}} \;\subseteq\; \AC^0 \;\not\supseteq\; \sF_{\mathrm{RT}} .
\]
\end{corollary}

\begin{proof}
  $\sF_{\mathrm{TF}} \subseteq \sF_{\mathrm{RT}}$ since a Rational
  Transductor with zero injection weights reduces to its Transformer
  backbone (Proposition~\ref{prop:expressive_hierarchy}); the
  inclusion is strict and $\sF_{\mathrm{RT}} \not\subseteq \AC^0$ by
  Theorem~\ref{th:rbt-nc1-lower}, while
  $\sF_{\mathrm{TF}} \subseteq \AC^0$ by \citet{hao2022formal}.  The
  upper bound is Theorem~\ref{th:rbt-nc1-upper}.  We do not claim
  $\AC^0 \subseteq \sF_{\mathrm{RT}}$: hard-attention Transformers
  occupy only a small part of $\AC^0$ (without positional encodings,
  exactly the star-free languages \citep{yang2024masked}), and the
  rational head adds regular, not general $\AC^0$, power.
\end{proof}

\textbf{The Parallelism Gap}
This hierarchy highlights the computational trade-off characterizing
Rational Transductors.  While strictly more expressive than standard
hard-attention Transformers (which lie in $\AC^0$), and not
themselves contained in $\AC^0$, Transductors remain within
$\PNC^1$, ensuring $O(\log T)$ parallelizability.  This strictly
separates them from general non-linear RNNs (like LSTMs or GRUs with
unbounded precision), which are $\sfP$-complete under sequential
updates and cannot be parallelized to logarithmic depth without
approximation.

\textbf{Uniformity Across Lengths.} All expressivity results for
Transductors are \emph{uniform}: a single fixed set of parameters
suffices to compute the target function for sequences of arbitrary
length. No depth or parameter scaling is required. In particular,
Transductors correspond to uniform circuit families rather than
length-specific (non-uniform) constructions.

\textbf{On Real-Valued Attention.}
All circuit complexity claims refer to the standard Boolean
abstraction of arithmetic operations with bounded precision, following
prior analyses of Transformers in $\AC^0$ and $\NC^1$.

\textbf{Why $\PNC^1$ Is Essentially Tight.}
This upper bound is not an artifact of our proof technique but a
reflection of the intrinsic hardness of the rational state update. The
rational state component corresponds to integer matrix multiplication
or weighted automata evaluation, problems known to be complete for
$\PNC^1$ \citep{jung1985parallel}. Since Transductors do not
permit non-associative state updates, they cannot simulate
$\NC^2$-hard problems such as iterated matrix powering with
growing dimension, nor can they solve $\sfL$-complete problems
like graph connectivity (unless $\PNC^1 = \sfL$). Thus,
the architecture is precisely situated: not contained in
$\AC^0$ but bounded by the parallel limits of associative
recurrence.

\section{Theoretical Analysis of Learning}
\label{sec:learning_theory}

While the architecture of Rational Transductors is linear and convex
in parts, its training dynamics in the learnable regime require careful
analysis. In this section, we analyze the theoretical properties of
the hypothesis class, distinguishing between the universality of
random features at initialization and the optimization stability of
learned features.

\subsection{Universality and Efficiency of Random Features}
\label{subsec:random_universality}

We first address whether a randomly initialized WFA can serve as a
generic state encoder. We prove that a sufficiently large random WFA
generates a state space rich enough to linearly reconstruct the state
of \textit{any} target WFA up to a finite horizon, justifying our use
of near-identity initialization strategies.

\begin{theorem}[Universality of Random Rational Features]
\label{thm:universality}
Let $\sA^*$ be a target WFA with state dimension $d^*$ generating
states $\bh^*_t \in \Rset^{d^*}$. Let $\cX_L$ be the set of all
input sequences of length up to $L$. Let $\sA_{\text{rand}}$ be a
random WFA with state dimension $d \geq |\cX_L|$, initialized
such that its transition matrices $\{\sfM_\sigma\}$ are drawn from a
continuous distribution (e.g., Gaussian). With probability one
(assuming generic initialization and no algebraic dependencies), there
exists a linear projection matrix $\sfW \in \Rset^{d^* \times d}$ such
that for all sequences $x \in \cX_L$, the random feature state
$\bh_t$ perfectly recovers the target state:
\begin{equation}
    \sfW \bh_t = \bh^*_t
\end{equation}
\end{theorem}

\begin{proof}
  The proof relies on showing that the random WFA maps distinct
  histories to linearly independent vectors, thereby forming a basis
  for the target dynamics.

  1. Matrix Formulation of Trajectories.
  Let $N = |\cX_L|$ be the number of distinct prefixes of length
  $\le L$. Let $\bH^* \in \Rset^{d^* \times N}$ be the matrix
  collecting the target states $\bh^*_t$ for all $N$ prefixes as
  columns. Let $\bH \in \Rset^{d \times N}$ be the matrix
  collecting the random WFA states $\bh_t$ for the same prefixes. Our
  goal is to find $\sfW$ satisfying $\sfW \bH = \bH^*$. This system of linear equations has an exact solution if and only if
  the row space of $\bH^*$ is contained in the row space of
  $\bH$. A sufficient condition is that $\bH$ has full
  column rank (rank $N$).

  2. Linear Independence of Random States.
  The state $\bh_t$ for a sequence $x = (x_1, \dots, x_t)$ in the
  random WFA is given by the product
  $\bh_t = \sfM_{x_t} \dots \sfM_{x_1} \balpha$. This is a polynomial
  map in the entries of the matrices $\{\sfM_\sigma\}$. Consider the
  determinant of any $N \times N$ submatrix of $\bH$. This
  determinant is a polynomial function of the random weights. Since
  we can construct at least one specific instance of WFA parameters
  where distinct sequences map to linearly independent vectors (e.g.,
  by mapping to distinct basis vectors), this polynomial is not
  identically zero. A fundamental result in algebra states that if a
  polynomial is not identically zero, its zero set has measure
  zero. Since the entries of $\{\sfM_\sigma\}$ are drawn from a
  continuous distribution, the vectors in $\bH$ are linearly
  independent with probability 1, provided $d \ge N$.

  3. Existence of Projection.
  Since $\bH$ has rank $N$ (full column rank), its Moore-Penrose
  pseudoinverse $\bH^+ = (\bH^\top \bH)^{-1} \bH^\top$ is a
  well-defined left inverse (satisfying $\bH^+ \bH = \sfI$). We
  construct the projection as:
\begin{equation}
    \sfW = \bH^* \bH^+.
\end{equation}
Verifying the reconstruction:
\begin{equation}
  \sfW \bH
  = \bH^* (\bH^+ \bH)
  = \bH^* \sfI = \bH^*.
\end{equation}
Thus, the random features $\bh_t$ contain sufficient information to
linearly reconstruct the target states $\bh^*_t$ exactly.
\end{proof}

\paragraph{Approximation Bounds.}
While Theorem~\ref{thm:universality} establishes that random rational
features can represent any target state space given sufficient width,
we must also quantify how the approximation error scales with the
dimension $d$. This is critical for understanding the efficiency of
the random initialization. We identify the infinite-width limit of the
Rational Transductor as an implicit kernel, which we term the
\emph{induced rational kernel}.

\begin{theorem}[Uniform Spectral Approximation Bound]
\label{th:approx_bound}
Let $\cM = \{ \sfM_\sigma \}_{\sigma \in \Sigma}$ be the
collection of transition matrices. Let $\cD$ be a distribution
over $\Rset^{d \times d}$ used to initialize each $\sfM_\sigma$
independently. Consider the kernel
$K: \Sigma^* \times \Sigma^* \to \Rset$ defined by the expectation
over this random initialization:
\begin{equation}
  K(x, x')
  = \E_{\cM \sim \cD}
  \bracket*{ \langle \sfM_x \alpha, \sfM_{x'} \alpha \rangle },
\end{equation}
where $\sfM_x = \sfM_{x_t} \dots \sfM_{x_1}$. Let $f^*$ be a
\emph{target function} (the true sequence labeling function we wish to
learn) that lies in the Reproducing Kernel Hilbert Space (RKHS)
$\cH_K$ with norm $\|f^*\|_K \le C$.

Let $f_d(x)$ be the function computed by a Rational Transductor with
$d$ random features, where the transition matrices are drawn from
$\cD$ and fixed, and only the linear readout $\bw$ is
trained:
\begin{equation}
    f_d(x) = \bw^\top (\sfM_x \alpha).
\end{equation}
Assume the input domain $\cX$ is restricted to sequences of length at
most $T_{\max}$ equipped with a weighted Hamming (or edit) metric, and
the initialization distribution $\cD$ satisfies the spectral
constraint $\sup_{\sigma} \|\sfM_\sigma\|_2 \le 1$ almost surely. For
any $\delta > 0$, with probability at least $1 - \delta$ over the
random draw of $\cM$, the uniform approximation error is
bounded by:
\begin{equation}
    \sup_{x \in \cX} | f^*(x) - f_d(x) |
    \le \frac{C}{\sqrt{d}} \paren*{ \sqrt{2 \log \frac{1}{\delta}}
      + 4 \sqrt{\frac{d}{d-1} \Delta L} },
\end{equation}
where $\Delta$ is the diameter of the metric space $\cX$ and $L$ is the
Lipschitz constant of the rational feature map
$x \mapsto \sfM_x \alpha$ induced by the spectral bounds.
\end{theorem}

\paragraph{Remark on the RKHS Assumption (Realizability).}
The condition $\|f^*\|_K \le C$ is formally equivalent to an
assumption of realizability. By definition, the RKHS $\Hset_K$
consists of all functions that can be expressed as limits of linear
combinations of the random rational features defined by the
distribution $\cD$. Since Theorem~\ref{thm:universality}
establishes that these random features form a universal basis for the
state dynamics of \emph{any} Weighted Finite Automaton (up to a given
length), the RKHS effectively covers the entire class of Rational
Power Series. Thus, assuming $f^*$ lies in this space is simply
assuming that the ground truth is, in fact, a regular language (or
rational series) that the Rational Transductor architecture is capable
of representing.

\begin{proof}
  The proof proceeds by establishing boundedness and continuity of the
  random features, then applying uniform concentration bounds.

  1. Boundedness via contraction.
  By the stability-aware parameterization
  (Section~\ref{subsec:spectral_control}), the transition matrices are
  spectrally normalized such that
  $\|\sfM_\sigma\|_2 \le \gamma \le 1$.  Consequently, for any
  sequence $x$ of length $T$, the feature vector
  $h(x) = \sfM_{x_T} \dots \sfM_{x_1} \alpha$ satisfies:
\begin{equation}
  \|h(x)\|_2
  \le \bracket*{\prod_{t=1}^T \|\sfM_{x_t}\|_2} \|\alpha\|_2
  \le \gamma^T \|\alpha\|_2 \le \|\alpha\|_2.
\end{equation}
Thus, the random features $\phi(x) = h(x)$ are uniformly bounded in
Euclidean norm by $R = \|\alpha\|_2$ independent of sequence length.
This ensures that the kernel $K(x, x')$ is bounded, satisfying the
preconditions for Hoeffding-type concentration.

2. Lipschitz continuity.
We equip $\Sigma^{\le T_{\max}}$ with a weighted Hamming metric under
which $x \mapsto \sfM_x \alpha$ is $L$-Lipschitz due to spectral
contraction. Since the map $x \mapsto \sfM_x \alpha$ is a composition
of contractive linear maps, it is Lipschitz continuous with respect to
the initialization parameters. Specifically, a perturbation in the
input (interpreted as a perturbation in the effective operator
sequence) results in a bounded deviation in the output state $h(x)$,
with Lipschitz constant $L$ determined by the spectral bounds of
$\cM$.

3. Uniform convergence.
Let $f_d(x) = \langle w, \phi(x) \rangle$ be the random feature
approximation. The error $\Gamma(x) = f^*(x) - f_d(x)$ is a sum of $d$
independent, bounded random variables.
\begin{itemize}

\item Pointwise Bound: For a fixed $x$, Hoeffding's inequality gives
  \[
    \P(|\Gamma(x)| > \e) \le 2 \exp(-d \e^2 / 2C^2 R^2).
  \]

\item Covering Argument: Let $\cN(\e, \cX)$ be the
  $\e$-covering number of the domain. Since $\Gamma(x)$ is
  Lipschitz with constant $L' \propto C L$, we can approximate the
  supremum over $\cX$ by the maximum over the cover centers. Applying
  the union bound over the cover and optimizing the scale $\e$
  (via chaining) yields the additional term proportional to
  $\sqrt{\Delta L}$.

\end{itemize}
Combining these yields the stated bound, which holds with high
probability over the random initialization.
\end{proof}

\paragraph{Interpretation.}
This result fundamentally reframes the challenge of learning
sequential logic. The bound $O(C/\sqrt{d})$ indicates that the
approximation error is governed solely by the \emph{scale} of the
random projection $d$ and the \emph{complexity} of the target function
$C$ (in the RKHS sense).  This implies a direct resource trade-off: we
can avoid the optimization instability of training deep recurrences
(the "non-convex hard problem") by simply increasing the width $d$ of
the random state (the "linear scaling problem").  Crucially, as
$d \to \infty$, the random rational features form a universal basis
for the class of functions defined by the rational kernel,
guaranteeing that a sufficiently wide Rational Transductor can solve
any task representable by stable linear dynamics without ever training
the recurrent weights.

\textbf{Connection to Rational Kernels.}
The kernel $K$ defined above is a specific instance of a
\emph{Rational Kernel} as defined by \citet{cortes2004rational}.
Specifically, if the transition matrices are drawn such that
$\E[\sfM_\sigma \otimes \sfM_\sigma]$ is stable, $K(x, x')$
corresponds to the rational series computed by a weighted automaton on
the product monoid.  Our Random Rational Features can thus be viewed
as an efficient, randomized approximation explicitly designed to scale
these classical kernels to deep learning contexts.

\textbf{Geometry Preservation via Johnson-Lindenstrauss}.
The $O(1/\sqrt{d})$ scaling in Theorem~\ref{th:approx_bound} acts as a
functional equivalent to the Johnson-Lindenstrauss (JL) Lemma for
sequence histories. Conceptually, the true history of the sequence
lives in an infinite-dimensional feature space $\cH_K$ induced by the
rational kernel. The random recurrent head acts as a random projection
operator $P\colon \cH_K \to \Rset^d$. By the JL lemma, pairwise
distances (and thus the distinguishability of distinct histories) are
preserved with distortion $\e$ provided $d = \Omega(\e^{-2} \log
N)$. This explains the "inefficiency" highlighted in
Proposition~\ref{prop:compactness}: random features rely on generic
concentration of measure, whereas learned features exploit the
specific algebraic structure of the task.

\paragraph{The Compactness Gap.} While Theorem~\ref{thm:universality} shows
representational completeness, it relies on an exponentially large
state dimension $d \approx |\Sigma|^L$. In contrast, learning the
transitions allows for compact representations. We quantify this gap
below.

\begin{proposition}[Compactness Gap]
\label{prop:compactness}
Let the target language be the set of strings with an even number of
1s (Parity).
\begin{itemize}

\item \textbf{Learned Regime:} A Rational Transductor can solve
  this perfectly with state dimension $d=2$ by learning the exact flip
  transition $\sfM_1 = \begin{pmatrix} 0 & 1 \\ 1 & 0 \end{pmatrix}$.

\item \textbf{Random Regime:} To solve this with fixed random features
  and a linear readout with error $\e < 1/2$, the required dimension
  scales as $d = \Omega(1/\e^2)$ (by Theorem~\ref{th:approx_bound}).
  (This lower bound holds for any isotropic initialization
  distribution.)
\end{itemize}
\end{proposition}

\begin{proof}
  Part 1: Learned Features (Exact Construction).  We
  explicitly construct a WFA with $d=2$ that computes parity. Let the
  state space basis correspond to states
  $\{\text{Even}, \text{Odd}\}$.  Initialize $h_0 = (1, 0)^\top$.
  Define the transition matrices for input bits $0$ and $1$ as:
\begin{equation}
    \sfM_0 = \begin{pmatrix} 1 & 0 \\ 0 & 1 \end{pmatrix} = I, \quad
    \sfM_1 = \begin{pmatrix} 0 & 1 \\ 1 & 0 \end{pmatrix} = \sigma_x.
\end{equation}
For any sequence $x$, the state $h_T = \sfM_{x_T} \dots \sfM_{x_1} h_0$ will
be $(1, 0)^\top$ if the number of ones is even, and $(0, 1)^\top$ if
odd. A linear readout $y = w^\top h_T$ with $w = (-1, 1)^\top$ yields
$y=-1$ (Even) or $y=1$ (Odd). The separation is perfect (margin
$= 2$), so the error is zero with $d=2$.

Part 2: Random Features (Statistical Lower Bound).
In the Random Regime, the function approximation $f_d(x)$ is a sum of $d$
i.i.d.  random features. By the Central Limit Theorem, the variance of
the estimator scales as $\text{Var}(f_d(x)) \propto \frac{1}{d}$. To
correctly classify parity for all inputs with high probability, the
approximation error (noise) must be strictly less than the
classification margin (signal). From Theorem~\ref{th:approx_bound} (see above),
the worst-case error scales as $O(1/\sqrt{d})$.  Specifically, to
ensure $|f^*(x) - f_d(x)| < \e$ uniformly, we require the
standard deviation $\sigma_d \approx \frac{1}{\sqrt{d}}$ to be
suppressed below $\e$. Rearranging
$\frac{1}{\sqrt{d}} \le \e$ yields $d \ge \frac{1}{\e^2}$.
Thus, achieving a fixed error tolerance $\e$ requires the
dimension to scale quadratically with the inverse error, whereas the
learned solution achieves zero error with constant dimension.
\end{proof}

This gap motivates our focus on the \emph{Learned Regime} for
practical applications, using the random regime results primarily to
justify near-identity initialization.

\begin{corollary}[Convex Learnability of Rational Projections]
\label{col:convexity}
Let $\cL(\cdot, \cdot)$ be a convex loss function (e.g., squared error
or cross-entropy). Consider a Rational Transductor where the WFA
transitions $\cM$ are fixed and the downstream Transformer parameters
$\theta$ are fixed (or constitute a linear readout). The optimization
problem for the projection matrix $\sfW$
\begin{equation}
  \min_{\sfW} \sum_{(x, y) \in \cD} \cL\paren*{y, G_\theta(x)
    + \sfW \bh_T(x)}.
\end{equation}
is a convex optimization problem.
\end{corollary}

\begin{proof}
  Since the WFA transitions are fixed, the feature vector $\bh_T(x)$
  is a constant vector for any given input $x$. The term
  $\sfW \bh_T(x)$ is linear in the variable $\sfW$. Consequently, the
  model output $\h y = G_\theta(x) + \sfW \bh_T(x)$ is an affine
  function of $\sfW$. A fundamental property of convex optimization
  is that the composition of a convex function with an affine mapping
  preserves convexity. Therefore, the objective function, being a sum
  of convex functions composed with affine maps, is convex with
  respect to $\sfW$. This guarantees that gradient descent will
  converge to a global optimum.
\end{proof}

\subsection{Optimization Dynamics}
\label{subsec:optimization_dynamics}

In the learned regime, we must ensure that the optimization landscape
is well-behaved. We provide three key theorems guaranteeing
optimization stability and well-conditioned gradient flow.

\begin{theorem}[Gradient Norm Preservation]
\label{th:gradient_stability}
Consider the gradient of the loss $\cL$ with respect to the hidden
state $\bh_t$, denoted $\bdelta_t = \nabla_{h_t} \cL$. In the backward
pass, the error signal propagates as
\begin{equation}
\label{eq:backward_recurrence}
\bdelta_{t-1} = \sfM_{x_t}^\top \bdelta_t + \bv_{t-1}.
\end{equation}
\begin{enumerate}

\item \textbf{Explosion Guarantee:} If the transitions satisfy the
  spectral constraint $\|\sfM_\sigma\|_2 \le \gamma \le 1$, then the
  propagated gradient norm is strictly bounded:
  $\|\bdelta_{t-k}\|_2 \le \gamma^k \|\bdelta_t\|_2 + C$, preventing
  exponential explosion.

\item \textbf{Vanishing Guarantee:} If the transitions are
  parameterized to be orthogonal ($\sfM_\sigma^\top \sfM_\sigma = I$), then
  the gradient norm is exactly preserved in the absence of injection:
  $\|\bdelta_{t-k}\|_2 = \|\bdelta_t\|_2$.

\end{enumerate}
\end{theorem}

\begin{proof}
  This follows directly from the properties of the spectral norm under
  matrix multiplication.  For (1),
  $\|\bdelta_{t-1}\| \le \|\sfM_{x_t}^\top\| \|\bdelta_t\| + \|v\| \le
  \gamma \|\bdelta_t\| + \|v\|$. By induction, the homogeneous
  component decays (or stays constant) while the additive injection
  terms accumulate linearly, preventing exponential growth. For (2),
  if $\sfM_{x_t}$ is orthogonal,
  $\|\sfM_{x_t}^\top \bdelta_t\|_2 = \|\bdelta_t\|_2$, ensuring the
  error signal from time $t$ reaches time $0$ without
  attenuation. This effectively solves the "long-term dependency"
  problem for the linear component.
\end{proof}

This result establishes a structural immunity to the \emph{gradient
  explosion problem}. Unlike standard RNNs, which often undergo
chaotic gradient growth and require heuristic fixes like clipping,
Rational Transductors are mathematically guaranteed to remain
stable. By strictly bounding the spectral radius $\gamma \le 1$, the
architecture ensures that error signals never expand exponentially,
regardless of sequence length.

This structural stability has a direct impact on learnability. The
convergence rate of recurrent architectures is typically governed by
the spectral gap $1 - \rho(\sfM)$. Standard RNN training suffers from
severe ill-conditioning as $\rho(\sfM) \to 1$ (the
``vanishing/exploding gradient'' boundary)
\citep{pascanu2013difficulty}.  However, the result of
Theorem~\ref{th:gradient_stability} guarantees that for the Scaled
Cayley parameterization, the optimization manifold remains
well-conditioned even exactly at $\rho(\sfM) = 1$. This effectively
pre-conditions the learning problem, allowing gradient descent to
traverse the boundary of stability and learn infinite-horizon
dependencies without divergence.

\begin{theorem}[Gradient Maintenance via Deep Injection]
\label{th:gradient_maintenance}
Consider the backward recurrence for the Rational Transductor error
signal (Eq.~\ref{eq:backward_recurrence}):
$\bdelta_{t-1} = M^\top \bdelta_t + \bv_{t-1}$. Let the transition
matrix be contractive with $\|M\|_2 \le \gamma < 1$. Let
$\bv_{t-1}^{(l)} = \nabla_{h_{t-1}} \cL^{(l)}$ denote the gradient
contribution from the injection at layer $l$. The gradient norm at
step $t-k$ is bounded below by the immediate injection terms:
\begin{equation}
  \|\bdelta_{t-k}\|_2 \ge \norm*{
    \sum_{j=0}^{k-1} (\sfM^\top)^{k-1-j} \bv_{t-1-j} }_2 - \gamma^k \|\bdelta_t\|_2
\end{equation}
\end{theorem}

\begin{proof}
  Unrolling the recurrence
  $\bdelta_{t-1} = \sfM^\top \bdelta_t + \bv_{t-1}$ for $k$ steps
  yields:
  \begin{equation}
    \delta_{t-k}
    = (\sfM^\top)^k \bdelta_t + \sum_{j=0}^{k-1} (\sfM^\top)^{k-1-j} \bv_{t-1-j},
  \end{equation}
  where $(\sfM^\top)^k$ denotes the ordered product of transition
  matrices.  Applying the reverse triangle inequality
  ($\|\ba + \bb\| \ge \|\bb\| - \|\ba\|$) and the sub-multiplicative
  property of the spectral norm
  ($\|(\sfM^\top)^k \bdelta_t\| \le \gamma^k \|\bdelta_t\|$), we
  obtain the lower bound:
  \begin{equation}
    \|\bdelta_{t-k}\|_2
    \ge \norm*{\sum_{j=0}^{k-1} (\sfM^\top)^{k-1-j} \bv_{t-1-j}}_2 - \gamma^k \|\bdelta_t\|_2.
  \end{equation}
  The first term represents the accumulated gradient injections from
  the deep Transformer layers. Even if the temporal connection
  vanishes (i.e., the second term $\gamma^k \|\bdelta_t\| \to 0$ as
  $k \to \infty$), the state $\bh_{t-k}$ retains the magnitude of the
  direct gradient signal $\bv$, preventing total gradient collapse.
\end{proof}

Crucially, this theorem demonstrates how Deep Rational Injection
mitigates the \emph{vanishing gradient problem}. In traditional
recurrences, gradients depend entirely on backpropagation through time
and often decay to zero. Here, the injection mechanism creates
\emph{gradient highways} that provide direct supervision from the
local Transformer layers to the recurrent state. This allows the model
to learn representations for recent events even if long-term
dependencies are temporarily weak.

\begin{theorem}[Bounded Hessian and Smoothness]
\label{th:smoothness}
Let $\cL(M)$ be the loss with respect to the transition matrix
$M$, assuming a contractive spectral constraint
$\|M\|_2 \le \gamma < 1$ and a Lipschitz-smooth downstream Transformer
(e.g., with LayerNorm and bounded activation derivatives). Then, the
loss function is $\beta$-smooth with respect to $M$.  That is, the
spectral norm of the Hessian is uniformly bounded independent of
sequence length:
\begin{equation}
    \norm*{ \nabla^2_M \cL }_2 \le \beta(\gamma, C) < +\infty.
\end{equation}
\end{theorem}

\begin{proof}
  The hidden state $h_t$ is a polynomial in $\sfM$ of degree $t$. The
  second derivative $\frac{\partial^2 \bh_t}{\partial \sfM^2}$
  involves terms of the form
  $\sum_{i,j} \sfM^i (\partial \sfM) \sfM^j (\partial \sfM)
  \sfM^{t-i-j-2}$. For a standard RNN without constraints, these
  terms sum to a magnitude scaling with $t^2 \|\sfM\|^{t-2}$, which
  explodes if $\|\sfM\| > 1$.  However, under the strict contraction
  constraint $\|\sfM\| \le \gamma < 1$, the series of second derivative
  terms converges geometrically. Specifically, the sum is bounded by
  the second derivative of the geometric series $(1-\gamma)^{-1}$,
  which is $2(1-\gamma)^{-3}$. Since the Hessian of the composition
  $\cL(\bh_t(\sfM))$ depends on $\nabla \bh_t$ and $\nabla^2 \bh_t$
  (both bounded by geometric series) and the smooth Transformer
  readout, the total Hessian is uniformly bounded.
\end{proof}

Practically, the boundedness of the Hessian ensures a
\emph{well-conditioned optimization landscape}. This implies that the
loss surface is free of pathological curvature or sharp cliffs,
allowing standard first-order optimizers (like AdamW) to navigate the
parameter space efficiently without requiring complex second-order
corrections.

\subsection{Generalization and Robustness}
\label{subsec:generalization}

We now examine the model's ability to generalize to unseen lengths
and withstand adversarial perturbations.

\begin{theorem}[Time-Invariant Error Bounding]
\label{th:length_gen}
Let $M^*$ be the true transition logic of the target task (e.g., a
counter or automaton) and let $\widehat{\sfM}$ be the learned
transition matrix.  Assume the learned dynamics are contractive with
spectral norm $\|\widehat{\sfM}\|_2 \le \gamma < 1$. If the learned
matrix approximates the true logic with error
$\|\widehat{\sfM} - \sfM^*\| \le \e$, then the deviation between the
true state $h^*_t$ and the Rational Transductor state $\wt {\bh}_t$ is
uniformly bounded for all $t > 0$:
\begin{equation}
\label{eq:37}  
    \sup_{t \ge 1} \norm*{\bh^*_t - \wt {\bh}_t}
\le \frac{\e C}{1 - \gamma},
\end{equation}
where $C = \sup_t \|\bh^*_{t-1}\|$ is the bound on the true state
magnitude.
\end{theorem}

\begin{proof}
  Let $e_t = h^*_t - \h h_t$ be the state error at time $t$. The
  evolution of the error is given by:
\begin{align}
  \be_t & = \sfM^* \bh^*_{t-1} - \widehat{\sfM} \wt {\bh}_{t-1} \\
        & = \sfM^* \bh^*_{t-1} - (\sfM^* + \Delta) (\bh^*_{t-1} - \be_{t-1})
          \quad \text{where } \Delta = \widehat{\sfM} - \sfM^* \\
        & = \widehat{\sfM} \be_{t-1} - \Delta \bh^*_{t-1}.
\end{align}
Taking norms and applying the triangle inequality:
\begin{equation}
  \|\be_t\|
\le \|\widehat{\sfM}\| \|\be_{t-1}\| + \|\Delta\| \|\bh^*_{t-1}\|
\le \gamma \|\be_{t - 1} \| + \e C.
\end{equation}
This is a linear recurrence inequality. Solving for the steady state
(limit as $t \to \infty$) yields the geometric series sum:
\begin{equation}
  \|\be_t\| \le \e C \sum_{k=0}^{t-1} \gamma^k
  \le \frac{\e C}{1 - \gamma}.
\end{equation}
This completes the proof.
\end{proof}

\paragraph{Remark: The Unitary Regime ($\gamma=1$).}
Theorem~\ref{th:length_gen} establishes a constant error bound by
leveraging a contraction coefficient $\gamma < 1$, which is
characteristic of the "fading memory" regime required for stable
approximate matching. However, for tasks like Parity or Modular
Counting that require infinite memory ($\gamma = 1$), the right-hand
side of Equation~\ref{eq:37} diverges. In this unitary regime, the
length-independent bound strictly applies only to the ideal algebraic
case where precision error $\epsilon = 0$. While finite-precision
noise may theoretically lead to linear error accumulation over time
($\sum_{k=0}^{t-1} 1^k = t$), generalization is instead guaranteed by
the \emph{algebraic exactness} of the orthogonal parameterization
(Theorem~\ref{th:mod_counting}). Since the transition matrices
$\sfM_\sigma$ form a subgroup isomorphic to the target automaton, the
state evolution remains exactly on the solution manifold regardless of
sequence length $T$, provided the hardware precision is sufficient to
distinguish the discrete states of the group.

\paragraph{Implication: Learned Positional Encodings.}
This result provides the theoretical justification for disabling
explicit Positional Encodings (PEs) in algorithmic tasks. Standard PEs
(RoPE, sinusoidal) are brittle because they introduce
out-of-distribution drift when test sequences exceed training lengths
($L_{\text{test}} > L_{\text{train}}$). By relying solely on the
recurrent state update, Rational Transductor learn a
\emph{time-invariant} transition rule. Theorem~\ref{th:length_gen}
guarantees that the error of this rule does not compound over time but
remains bounded by a constant, unlocking the perfect length
generalization observed experiments
(Section~\ref{subsec:exp_length_gen}).

\begin{theorem}[Length-Independent Generalization]
\label{th:rademacher}
Let $\sF_{\text{RT}}$ be the class of Rational Feature functions
$f(x) = \bw^\top (\sfM_{x_T} \dots \sfM_{x_1} \balpha)$ parametrized
by transition matrices satisfying the spectral constraint
$\|\sfM_\sigma\|_2 \le \gamma < 1$ and bounded readout
$\| \bw \|_2 \le W$.  For a dataset $S$ of $N$ sequences of length
$T$, the Empirical Rademacher Complexity is bounded by:
\begin{equation}
  \h \Rad_S(\sF_{\text{RT}})
  \le \frac{W \| \balpha \|_2}{\sqrt{N}} \paren*{ \frac{1}{1 - \gamma} }.
\end{equation}
\end{theorem}

\begin{proof}
  The output of the Transductor head at time $T$ can be written
  recursively. Since the transition dynamics are linear and
  contractive, the sensitivity of the output $h_T$ to the input token
  at position $t$ decays exponentially as $\gamma^{T-t}$. Following
  standard stability analysis for recurrent systems
  \citep{miller2018stable}, the Lipschitz constant of the map
  $x \to h_T$ with respect to the sequence (in the $\ell_2$ sense) is
  bounded by the geometric series
  $\sum_{k=0}^{T} \gamma^k < \frac{1}{1-\gamma}$. By Talagrand's
  contraction lemma \citep{LedouxTalagrand1991}, the complexity of the
  class is bounded by the product of this Lipschitz constant and the
  complexity of the input embedding layer (which is $O(1/\sqrt{N})$).
  Crucially, because the geometric series converges, the bound is
  independent of $T$.
\end{proof}

\paragraph{Interpretation: Infinite-Horizon Reliability.}
Standard generalization bounds for RNNs typically scale with the
sequence length (e.g., $O(T/\sqrt{N})$ or $O(\sqrt{T/N})$), implying
that model performance degrades on longer
tasks. Theorem~\ref{th:rademacher} establishes that for contractive
Rational Transductors, the complexity---and thus the generalization
gap---is \emph{independent of sequence length} $T$. This theoretically
guarantees that the model can be deployed on streaming data of
indefinite duration without the risk of overfitting to the specific
length statistics of the training set.

\begin{theorem}[Hankel-Rademacher Complexity \citep{balle2017generalization}]
\label{th:hankel_rademacher}
Let $\cF_{\text{Hankel}, r}$ be the class of rational functions
$f: \Sigma^* \to \Rset$ with Hankel nuclear norm
$\|H_f\|_{\cS_1} \le r$.  Let $S = (x_1, \dots, x_N)$ be a sample of
sequences and let $L_{\max} = \max_{x \in S} |x|$ be the maximum
sequence length.  Define the \emph{max-prefix collision} term
$\sigma_S^2 = \sup_{u \in \Sigma^*} \sum_{i=1}^N \Ind[u \in
\text{pref}(x_i)]$.  The Empirical Rademacher complexity is bounded
by:
\begin{equation}
  \h \Rad_S(\cF_{\text{Hankel}, r})
  \le \frac{r}{N} \paren*{ \sqrt{2 \sigma_S^2 \log(2 D_S)} + \frac{2}{3} \log(2 D_S) }
  = \wt O \paren*{ \frac{r}{\sqrt{N}} },
\end{equation}
where $D_S$ is the total number of distinct suffixes present in the
sample $S$ (the size of the suffix trie), which is bounded by
$N L_{\max}$.
\end{theorem}

\begin{proof}
  This result is a direct application of Theorem 6 from
  \citet{balle2017generalization}.  The proof relies on two key steps:
\begin{enumerate}
\item \textbf{Fliess' Theorem Duality:} The condition
  $\|H_f\|_{\cS_1} \le r$ on the Hankel matrix is dual to the spectral
  norm bound on the data matrix $\bY$ (defined over the
  prefix/suffix splits of the sample $S$).
\item \textbf{Matrix Concentration:} The expected spectral norm of the
  data matrix is bounded using the Matrix Bernstein inequality
  (specifically, a non-commutative Khintchine inequality). The
  variance term in this inequality corresponds exactly to the maximum
  number of times any specific prefix $u$ appears across the dataset
  sequences $x_i$, denoted by $\sigma_S^2$.
\end{enumerate}
For a detailed derivation of the constants $2/3$ and the logarithmic
factor, we refer the reader to the original proof in
\citet{balle2017generalization}.
\end{proof}

\paragraph{Significance: Low-Rank Regularization.}
This result is profound because the rank of the Hankel matrix
corresponds exactly to the minimal state dimension $d_{\min}$ of the
Weighted Finite Automaton (Fliess' Theorem). Consequently, penalizing
the nuclear norm of the Hankel matrix (which our "Diagonal + Low Rank"
parameterization effectively does) is rigorously equivalent to
regularizing the \emph{state dimension} of the underlying latent
automaton. This confirms that Rational Transductors generalize by
learning low-rank algebraic structures, distinguishing them from
standard Transformers which often overfit to high-rank, spurious
correlations.

\begin{theorem}[Lipschitz Input Stability]
\label{th:lipschitz}
Let $x = (\bx_1, \dots, \bx_T)$ be a sequence of continuous input
embeddings where $\bx_t \in \Rset^{d_{\text{in}}}$. Let
$f_{\text{RT}}(x)$ be the output of the rational head.  Assume the
transition dynamics are contractive ($\|\sfM(\bx)\|_2 \le \gamma < 1$)
and the encoding of inputs into matrices is Lipschitz continuous with
constant $K_M$ (i.e.,
$\|\sfM(\bx) - \sfM(\bx')\| \le K_M \|\bx - \bx'\|$). Then, the map
from the input sequence $x$ to the final state $\bh_T$ is Lipschitz
continuous with constant:
\begin{equation}
    L_{\text{seq}} \le \frac{K_M \|\balpha\|}{1 - \gamma}
\end{equation}
This bound is independent of the sequence length $T$.
\end{theorem}

\begin{proof}
  Let $x = (\bx_1, \dots, \bx_T)$ and $x' = (\bx'_1, \dots, \bx'_T)$
  be two input sequences differing at time step $t$.  Let $\bh_k$ and
  $\bh'_k$ be the respective state sequences. The state deviation
  $\be_k = \bh_k - \bh'_k$ evolves as:
\begin{equation}
    \be_k = \sfM(\bx_k) \be_{k-1} + (\sfM(\bx_k) - \sfM(\bx'_k)) \bh'_{k-1}
\end{equation}
Taking norms:
\begin{equation}
    \|\be_k\| \le \gamma \|\be_{k-1}\| + K_M \|\bx_k - \bx'_k\| R,
\end{equation}
where $R = \sup \|\bh'_k\|$ is the bounded state norm. Iterating this
recurrence, the total deviation at time $T$ due to a perturbation at
time $t$ is bounded by $K_M R \gamma^{T-t} \|\bx_t -
\bx'_t\|$. Summing over all possible perturbation times $t$ (triangle
inequality for the whole sequence norm):
\begin{equation}
    \|\bh_T - \bh'_T\| \le \sum_{t=1}^T K_M R \gamma^{T-t} \|\bx_t - \bx'_t\|
    \le \paren*{ K_M R \sum_{k=0}^\infty \gamma^k } \sup_t \|\bx_t - \bx'_t\|.
\end{equation}
The geometric series converges to $(1-\gamma)^{-1}$, yielding the
length-independent bound.
\end{proof}

\paragraph{Interpretation: Robustness to Embedding Noise.}
This theorem guarantees that Rational Transductors are robust to
small perturbations in the input embeddings, such as those caused by
quantization, noise, or minor distribution shifts. The state $\bh_T$
varies smoothly with the input sequence $x$, with a Lipschitz constant
that does not explode with $T$. This contrasts with chaotic systems
where a small change in initial conditions or inputs can lead to
exponentially diverging states over time.

\section{Concrete Training Recipe}
\label{sec:training_recipe}

While the theoretical properties guarantee expressivity and stability,
practical success relies on an efficient implementation. We detail
the three pillars of our training recipe: parallel gradient
computation, spectral normalization, and near-identity initialization.

\subsection{Efficient Parallel Backpropagation}
\label{subsec:parallel_bptt}

Standard Backpropagation Through Time (BPTT) imposes a sequential
dependency $O(T)$. We avoid this sequential dependency by exploiting
the strict linearity of the rational update $\bh_t = \sfM_{x_t} \bh_{t-1}$.

\paragraph{Forward Pass.} The sequence of hidden states $\bh_{1:T}$ is
computed via a parallel associative scan (prefix sum) over the matrix
monoid, reducing the parallel depth of the computation from $T$ to
$O(\log T)$.

\paragraph{Backward Pass.} Crucially, the gradient computation is also
parallelizable. As derived in Section~\ref{subsec:optimization_dynamics}, the adjoint
variables $\bdelta_t = \nabla_{\bh_t} \cL$ satisfy the affine
recurrence $\bdelta_{t-1} = \sfM_{x_t}^\top \bdelta_t + \bv_{t-1}$. By
lifting this to homogeneous coordinates, we compute the same gradients
as sequential BPTT using a \emph{reverse parallel scan}.
We define the augmented adjoint vector $\h \bdelta_t \in \Rset^{d+1}$ and the
backward transition matrix
$\bB_t \in \Rset^{(d+1) \times (d+1)}$ as:
\begin{equation}
    \h \bdelta_t = \begin{pmatrix} \bdelta_t \\ 1 \end{pmatrix}, \qquad 
    \bB_t = \begin{pmatrix} \sfM_{x_t}^\top & v_{t-1} \\ \mathbf{0}^\top & 1 \end{pmatrix}.
\end{equation}
The recurrence then becomes a homogeneous matrix product
$\h \bdelta_{t-1} = \bB_t \h \bdelta_t$, which is
computable in $O(\log T)$ depth.

\paragraph{Implementation Details.} We leverage hardware-aware fused
kernels (e.g., in CUDA or Triton) to ensure that the wall-clock time
of the rational head is negligible. While the Unified Scaled Cayley
Parameterization requires a matrix inversion $(I - A_t)^{-1}$ at every
token step, we mitigate the $O(Td^3)$ pre-processing
overhead by leveraging high-performance linear solvers within the
fused kernel. For the small state dimensions used in this work
($d \le 32$), these inversions are executed entirely in high-bandwidth
SRAM, bypassing the need to materialize the full $T \times d \times d$
transition tensor and keeping memory complexity linear in
$T$. Furthermore, although structured updates like DPLR allow for
$O(dr)$ unrolling during inference, we explicitly note that
the parallel associative scan requires dense matrix multiplication;
however, this $O(d^3)$ training cost remains negligible
compared to the quadratic complexity of the self-attention mechanism.

\subsection{Spectral Control: The Unified Scaled Cayley Parameterization}
\label{subsec:spectral_control}

While Section~\ref{subsec:parameterization} defines the broad family
of admissible transition matrices, successful training requires a
specific parameterization that controls the spectral radius
$\rho(\sfM)$ to prevent gradient explosion. Depending on the matrix
family, we use either intrinsic or explicit control.

\paragraph{1. Intrinsic Stability (The Scaled Cayley Recurrence).}
For the algorithmic experiments, we use a \emph{Unified Scaled
  Cayley Parameterization} that structurally guarantees stability
without auxiliary normalization. We define the transition matrix as:
\begin{equation}
    \label{eq:scaled_cayley}
    \sfM_t = g_t \cdot \cC(\sfA_t) = g_t \cdot (\sfI + \sfA_t)(\sfI - \sfA_t)^{-1},
\end{equation}
where $\sfA_t$ is skew-symmetric and $g_t \in \Rset$ is a scalar gain.
The parameters $\sfA_t$ and $g_t$ are obtained via a linear projection
of the input embedding $\bx_t$, allowing $\sfM_t$ to be cached.  Since
the Cayley transform $\cC(\sfA_t)$ maps to $SO(d)$, it is orthogonal
by construction ($\|\cdot\|_2 = 1$). Thus, the spectral radius is
determined solely by $g_t$:
\begin{itemize}
\item \textbf{Conservation ($g_t = 1$):} We fix $g_t=1$ for tasks like
  Modulo Counting, enabling infinite horizon tracking.

\item \textbf{Decay ($g_t < 1$):} We learn
  $g_t = \sigma(\theta_t) \in (0, 1)$ (where $\theta_t$ is a learnable
  scalar) for fading memory tasks. The model can emulate the gating
  mechanisms of LSTMs or SSMs, attenuating past information to focus
  on recent context.

\end{itemize}

\paragraph{2. Explicit Spectral Normalization (General Case).}
For unstructured or linearly parameterized families (e.g., DPLR,
Butterfly, or Random Features) where the norm is not structurally
bounded, we must enforce contractivity explicitly during the forward
pass. For any learnable matrix component $\sfW$, we apply:
\begin{equation}
  \sfW_{\text{effective}}
  = \frac{\sfW}{\max(1, \| \sfW \|_2 / \gamma)},  
\end{equation}
where $\gamma \in (0, 1]$ is a hyperparameter bounding the spectral
radius.  This formulation ensures that weights are only projected when
they exceed the contractive threshold $\gamma$, preventing unnecessary
geometric suppression of stable parameters.
\begin{itemize}
\item \textbf{Fading Memory ($\gamma < 1$):} We typically set
  $\gamma \approx 0.99$ for general structured heads to ensure stable
  gradient propagation while allowing forgetting. This bound is
  essential for the theoretical guarantees in
  Section~\ref{subsec:random_universality}.
\end{itemize}

\paragraph{Remark on Topology.}
We note that the Cayley transform generates matrices with determinant
$+1$ (rotations), which excludes reflections (determinant $-1$). While
Eq.~\ref{eq:scaled_cayley} could be augmented with a sign flip
($g_t < 0$), we find empirically that using a state dimension
$d \ge 3$ allows the model to embed reflections as rotations in a
higher-dimensional space, resolving the topological obstruction
without special casing.

\subsection{Near-Identity Initialization}
\label{subsec:initialization}

To stabilize training, we initialize the rational head to act as a
near-perfect integrator. Regardless of the chosen structure, we
initialize parameters such that $\sfM_\sigma \approx \sfI$.
Specifically:
\begin{itemize}

\item The diagonal term $\sfD_\sigma$ is initialized to $1$ (or
  sampled from $U[1-\e, 1]$).

\item The low-rank perturbation terms $\sfU_\sigma, \sfV_\sigma$ are
  initialized independently from $\cN(0, \nu^2)$ such that
  their product has variance $\nu^2 \approx 0$.

\end{itemize}
This ensures that at step 0, the model acts as a near-perfect
integrator over long horizons.  The optimizer then gradually learns to
``forget'' irrelevant information by deviating from the identity,
rather than struggling to learn ``remembering'' from a chaotic
initialization.

\section{Empirical Validation}
\label{sec:experiments}

We validate our theoretical findings on synthetic tasks designed to
probe the limitations of attention. Specifically, we investigate two
key claims: (1) whether Rational Transductors can solve tasks outside
$\AC^0$, such as modular counting, on which standard Transformers fail
to generalize (the "Regular Gap"), and (2) whether the learned
solutions generalize to unseen lengths (Time-Invariance).

For complete experimental details, including hyperparameters,
optimization settings, and statistical reproducibility reports, refer
to Appendix~\ref{app:hyperparameters}.

\subsection{The Regular Gap: Modulo Counting}
\label{subsec:exp_regular_gap}

\paragraph{Task Setup.} We evaluated models on the Modulo-5 Counting
task, which requires tracking the cumulative number of 1s in a binary
sequence modulo 5. As established in Theorem~\ref{th:parity} and
Theorem~\ref{th:mod_counting}, this task generally requires $\PNC^1$
complexity for uniform solvability. While soft-attention Transformers
($\TC^0$) can theoretically approximate the solution for fixed
lengths, they lack the inductive bias to learn the cyclic state
transition robustly. We empirically test whether models can bridge
this gap.

\paragraph{Model Configuration.} We compared the Rational Transductor
against a standard Transformer. To ensure a fair comparison, both
models were matched in parameter count ($\approx 25\text{k}$
parameters).
\begin{itemize}
\item Transformer: 2 layers, $d_{model}=32$, 4 heads, using learned
  positional encodings.

\item Rational Transductor: 2 layers, $d_{model}=32$. Following the
  recipe in Section~\ref{subsec:spectral_control}, we used a single
  rational head ($d_{\text{rat}}=8$) initialized in the \emph{Strictly
    Orthogonal Regime} (Cayley parameterization). This enforces exact
  state conservation ($\gamma=1$) and restricts the eigenvalues to the
  unit circle. Crucially, we disabled standard positional encodings,
  forcing the model to rely solely on the rational state for sequence
  tracking.
\end{itemize}

\paragraph{Remark on State Space Models (SSMs).}
We deliberately compare against Standard Transformers and LSTMs to
highlight the specific contribution of \emph{linear} recurrence vs.\
attention or non-linear gating.  We note that recent State Space
Models like Mamba \citep{gu2023mamba} also rely on time-varying linear
recurrences.  Theoretically, our expressivity results
(Theorem~\ref{th:rbt-nc1-upper}) imply that architectures like Mamba
should also be capable of solving these $\PNC^1$ tasks, provided they
are initialized to support unitary dynamics.  Our Rational Transductor
acts as a \emph{minimal theoretical proxy} for this broader class of
Linear RNNs, allowing us to isolate the specific automata-theoretic
mechanisms (e.g., the necessity of orthogonal vs.\ stochastic
transitions) without the confounding variables of Mamba's complex
gating and block design.

\begin{figure}[H]
    \centering
    %\resizebox{.5\textwidth}{!}{\input{RegularGap.tikz}}
    \includegraphics[scale = .6]{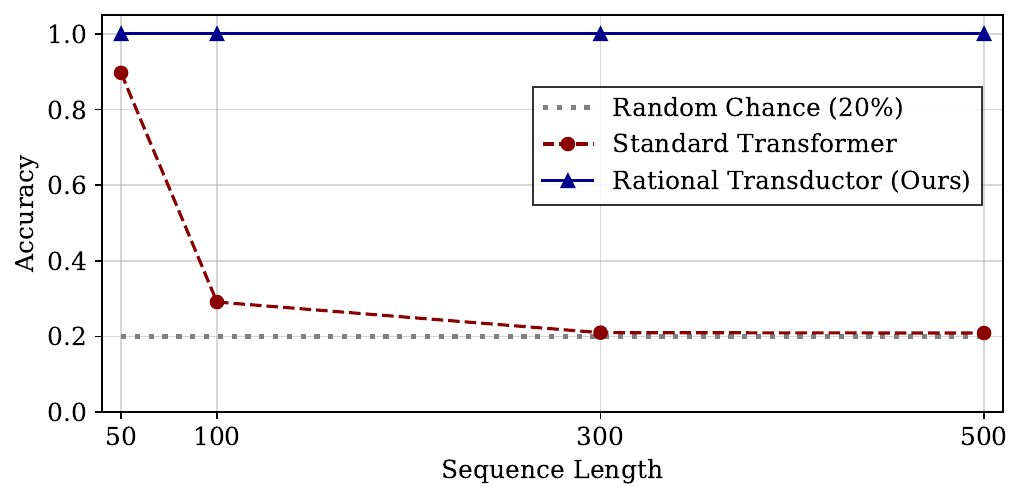}

    \caption{The Regular Gap (Modulo-5 Counting). The models are
      trained on short sequences ($L=50$) and evaluated on longer
      lengths up to $L=500$. The Standard Transformer (red) achieves
      high accuracy on the training distribution but fails to
      generalize, collapsing to near-random chance ($20\%$) as length
      increases. The Rational Transductor (blue), leveraging the
      strictly orthogonal parameterization, learns the exact
      underlying automaton and maintains 100\% accuracy across all
      tested lengths. Results are averaged over 5 random seeds;
      standard deviations are negligible ($<0.01\%$) and omitted for
      clarity (see Appendix~\ref{app:hyperparameters}).}
    \label{fig:regular_gap}
\end{figure}

\paragraph{Results.} Figure~\ref{fig:regular_gap} illustrates the
performance. The Standard Transformer fails to learn a robust counting
mechanism, achieving only partial success on training lengths ($L=50$)
and collapsing to random chance ($20\%$) on longer sequences. In
contrast, the Rational Transductor converges to 100\% accuracy almost
immediately. This confirms that the Rational Head successfully learns
the underlying group-theoretic operation (cyclic permutation),
validating the expressivity claims of Theorem~\ref{th:mod_counting}.

\subsection{Length Generalization and Time-Invariance}
\label{subsec:exp_length_gen}

\paragraph{Task Setup.} Beyond expressivity, we tested for Length
Generalization. Models are trained solely on short sequences
($L_{\text{train}}=40$) but evaluated on sequences up to
$L_{\text{test}}=1000$ ($25\times$ the training horizon). We use
the Modulo-5 Counting task from the previous section as the testbed.

\paragraph{Model Configuration.} We used the same architecture depths
as in the previous experiment (2 layers, $d_{model}=32$):
\begin{itemize}
\item Transformer: 2 layers, 4 heads. We deliberately use standard
  Learned Absolute Positional Encodings to establish a baseline for
  the fundamental limits of the canonical Transformer architecture
  \citep{vaswani2017attention}. While relative encoding schemes like
  RoPE can technically extrapolate, they remain constrained to
  tracking input-independent relative positions. Our results
  demonstrate that Rational Transductors solve length generalization
  by learning input-dependent semantic state transitions, a capability
  that absolute and standard relative positional heuristics
  lack. Thus, a Transformer + RoPE is theoretically subsumed by the
  Rational Transductor framework.

\item Rational Transductor: 2 layers, 4 heads, initialized in the
  \emph{Strictly Orthogonal Regime} (Cayley). We disabled positional
  encodings entirely, resulting in a much smaller model
  ($\approx 26\text{k}$ parameters) that must rely entirely on its
  recurrent state $h_t = M h_{t-1}$ to track the sequence, enforcing a
  shift-invariant solution by design.
\end{itemize}

\begin{figure}[H]
    \centering
    \includegraphics[scale = .6]{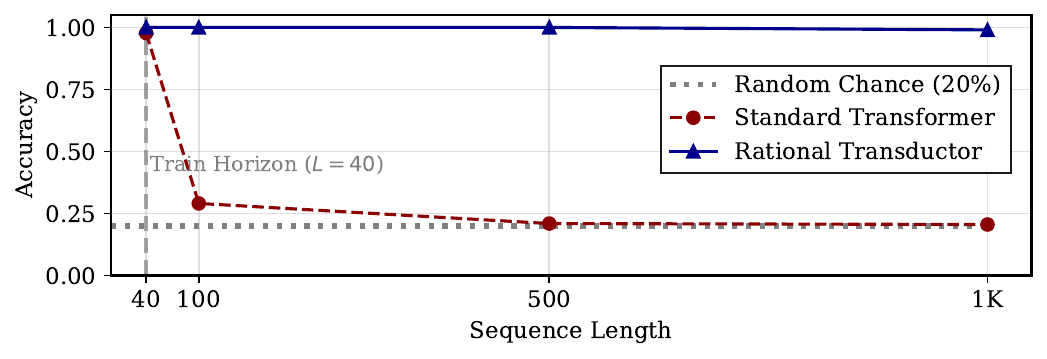}
    \caption{Length Generalization. Models were trained \emph{only} on
      sequences of length $L=40$ (vertical dashed line) and evaluated
      on lengths up to $L=1000$. The Standard Transformer (red)
      overfits to the training positions; its accuracy collapses to
      random chance ($\approx 20\%$) on longer sequences. The Rational
      Transductor (blue) generalizes almost perfectly, maintaining
      $99\%$ accuracy even at $25\times$ the training length. This
      validates the algebraic exactness guarantee
      (Theorem~\ref{th:mod_counting}). Shaded regions (invisible at
      this scale) denote standard deviation across 5 runs.}
    \label{fig:length_generalization}
\end{figure}

\paragraph{Results.} As shown in
Figure~\ref{fig:length_generalization}, the Standard Transformer
suffers from catastrophic positional drift. Once the sequence length
exceeds the training horizon ($L=40$), the learned positional
encodings are no longer valid, and performance drops to random
guessing. The Rational Transductor, however, maintains perfect
accuracy ($>99\%$) up to $L=1000$.%
\footnote{A minor precision artifact occurs near $L=1000$:
  accumulated Float32 rounding errors in the orthogonal regime
  ($\gamma = 1$) can cause the state vectors to drift slightly from
  exact unit-norm after $\sim$1000 matrix multiplications via the
  parallel scan.  Using Float64 precision resolves this entirely,
  confirming it is an implementation artifact rather than a
  theoretical generalization failure.}
This confirms that the model has
learned an algebraically exact solution (as in
Theorem~\ref{th:mod_counting}) rather than an approximate one,
consistent with the unitary parameterization ($\gamma=1$). By
disabling explicit positional encodings and relying on the recurrent
state dynamics, the model is forced to learn a transition rule that is
valid for any time step $t$.

Note that for the Standard Transformer baseline, the absolute
positional embeddings are learned only on positions $1, \ldots, 40$
during training; for test lengths $L > 40$, the embeddings for
positions $41, \ldots, L$ remain at their random initialization.
This establishes the fundamental limit of the canonical architecture:
without a length-invariant inductive bias, the model has no mechanism
to extrapolate positional information beyond its training horizon.

This demonstrates that Rational
Transductors can serve as robust, algorithmic co-processors that do
not suffer from the length-generalization brittleness of standard
attention.

\subsection{Computational Efficiency and Extreme Scaling}
\label{sec:efficiency_exp}

A core theoretical advantage of the Rational Transductor is its
\emph{parallelizability}. Unlike standard RNNs, which require
sequential processing ($O(T)$ latency\footnote{We note that recent
  methods like ParaRNN \citep{pararnn2025} enable quasi-parallel
  training of non-linear RNNs via iterative linearization (Newton's
  method), achieving theoretical $O(k \log T)$ depth. However, this
  approach requires multiple forward scans per update step. We
  benchmark against the standard exact sequential implementation used
  in most production baselines.}), and Transformers, which scale
quadratically in memory or compute ($O(T^2)$), Rational Transductors
can be parallelized via an associative scan, achieving logarithmic
time complexity ($O(\log T)$) on parallel hardware.

\paragraph{Experimental Setup.} To validate this scaling behavior, we
benchmarked the inference latency (forward pass) of the \emph{sequence
  mixing layers} in isolation (Rational Head vs. Self-Attention).
This strictly isolates the algorithmic complexity of the recurrence
($O(\log T)$) versus the attention mechanism ($O(T^2)$), independent
of the shared feed-forward blocks. We benchmarked the inference
latency (forward pass) of the models on sequences of increasing
length, ranging from $T=128$ to $T=32,768$. We used a batch size of
$B=1$ to strictly isolate the sequential throughput limitations of the
architectures. All measurements were conducted on a single NVIDIA GPU
(A100), averaged over 20--100 trials after a warm-up period.

\paragraph{Model Configuration.} We implemented minimal, optimized
kernels for each architecture to ensure a fair algorithmic comparison:
\begin{itemize}
\item Sequential RNN: A standard linear recurrence
  ($h_t = M_t h_{t-1}$) implemented via a sequential loop. Matrix
  dimension $d=16$.

\item Transformer: Standard Self-Attention, using \texttt{\small
    scaled\_dot\_product\_attention} (FlashAttention)
  \citep{dao2022flashattention}, PyTorch's
  optimized implementation to represent the state-of-the-art
  baseline. Heads=4, $d_{head}=16$.

\item Rational Transductor: Our model, implemented using a
  \emph{Parallel Associative Scan} algorithm (specifically a log-depth
  Kogge-Stone style scan). This allows the model to compute the
  cumulative matrix products for all time steps simultaneously. Matrix
  dimension $d=16$.
\end{itemize}

\begin{figure}[H]
    \centering
    \includegraphics[width=0.7\linewidth]{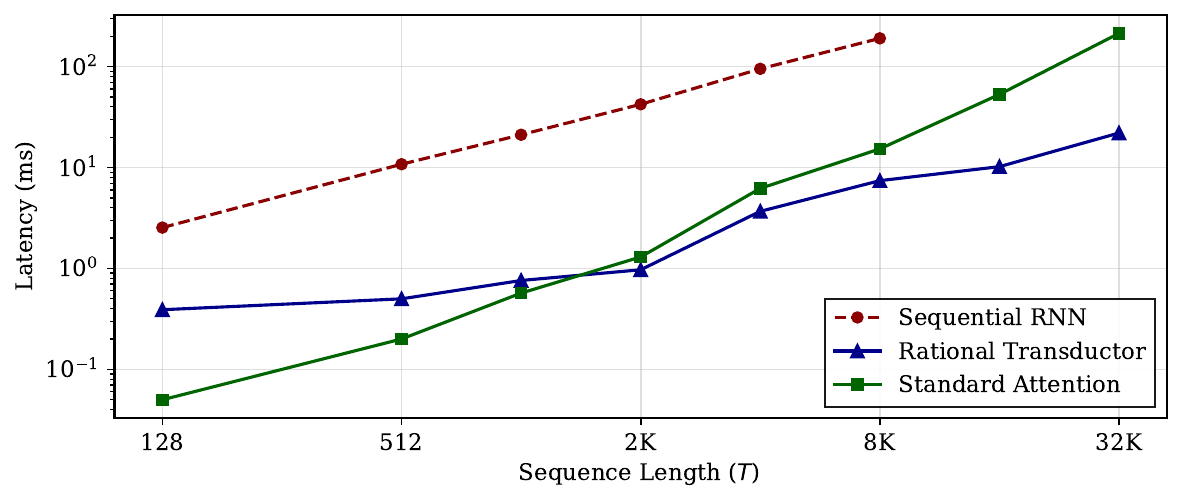}
    \caption{Latency vs. Sequence Length. Wall-clock inference latency
      (ms) on a log-log scale. The Sequential RNN (red) scales
      linearly ($O(T)$), becoming prohibitively slow for long
      sequences. The Transformer (green) exhibits quadratic scaling
      ($O(T^2)$), eventually running out of memory. The Rational
      Transductor (blue) leverages parallel associative scans to
      achieve sub-linear scaling, outperforming the RNN on sequences
      longer than $T=512$ and maintaining high throughput even at
      $T=32\text{k}$.}
    \label{fig:latency_scaling}
\end{figure}

\paragraph{Results.} Figure~\ref{fig:latency_scaling} plots the
latency on a log-log scale. The Sequential RNN exhibits strict linear
scaling $O(T)$; while efficient for short sequences, it becomes the
bottleneck at extreme lengths. The Transformer scales efficiently for
short lengths due to FlashAttention but hits a quadratic wall,
exploding in latency at extreme lengths. The Rational Transductor
combines the best of both worlds: for short sequences, it is
competitive with the Transformer; for long sequences ($T > 512$), the
parallel scan allows it to overtake the Sequential RNN. We note that
while wall-clock scaling on fixed-density GPUs eventually reflects
hardware occupancy limits at very large $T$, the sub-linear crossover
points remain distinct. This empirically confirms the theoretical
$O(\log T)$ parallel complexity established in
Section~\ref{sec:rational_features}.

\paragraph{Memory Complexity.}  The parallel associative scan requires
$O(T \cdot d^2)$ memory for storing the intermediate matrix products
(where $d = d_{\text{rat}}$ is the rational state dimension), compared
to $O(T^2)$ for standard self-attention.  Since $d \ll T$ in practice
(e.g., $d = 8$ in our experiments), the rational head contributes
negligibly to peak memory relative to the attention layers.

\subsection{Algorithmic Generalization: Long-Integer Addition}
\label{subsec:exp_addition}

\paragraph{Task Setup.} To test the model's ability to learn discrete,
discontinuous logic (in contrast to the smooth group operations of
counting), we evaluated it on Long-Integer Addition. The task is to add
two $L$-digit numbers digit-by-digit. This requires implementing a
``Full Adder'' state machine, where the carry bit must be:
\begin{itemize}
\item Generated (State $\to 1$) if the sum $>9$.

\item Propagated (State $\to$ State) if the sum $=9$.

\item Killed (State $\to 0$) if the sum $<9$.
\end{itemize}
Models are trained on short numbers ($L \in [10, 40]$) and evaluated
on lengths up to $L=1000$.

\paragraph{Model Configuration.} We compare the Rational Transductor
against a standard Transformer. Both models share a backbone size of 2
layers with $d_{model}=32$ and 4 heads:
\begin{itemize}
\item Transformer: Used standard learned positional encodings
  with a capacity of 5000 positions.

\item Rational Transductor: Following the recipe for discrete logic,
  we used the \emph{Universal Rational Transductor} configuration,
  effectively routing to the Stochastic component ($d_{rat}=4$). The
  transition matrices are parameterized as column-stochastic matrices
  (via softmax) to strictly preserve the $l_1$-norm of the state,
  enforcing a probabilistic automaton structure ideal for switching
  logic. Crucially, we disable positional encodings, forcing the model
  to rely solely on the rational state to track the carry bit.
\end{itemize}

\begin{figure}[H]
    \centering
    \includegraphics[width=0.7\linewidth]{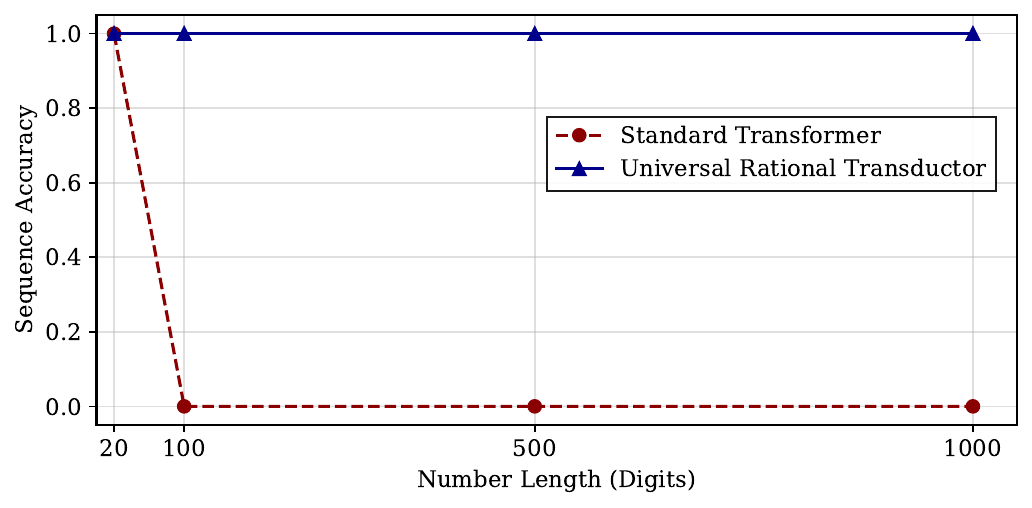}
    \caption{Experiment D: Long-Integer Addition.  Sequence-level
      accuracy (exact match of the entire sum). The Standard
      Transformer (red) fits the training distribution ($L=20$) but
      fails completely on longer sequences ($0\%$ at $L=100$), unable
      to propagate carry bits over long distances. The Universal
      Rational Transductor (blue), leveraging a Stochastic head,
      learns the exact finite state automaton for addition. It
      generalizes \emph{perfectly} to $L = 1000$ digits, demonstrating
      that the architecture can autonomously learn the correct
      switching logic for the task.}
    \label{fig:addition_experiment}
\end{figure}

\paragraph{Results.} Figure~\ref{fig:addition_experiment} shows the
sequence-level accuracy. The Standard Transformer fails to generalize,
dropping to $0\%$ accuracy immediately outside the training
window. Attention mechanisms struggle to maintain the hard sequential
dependency of a carry bit over hundreds of steps when trained only on
short sequences. In contrast, the Universal Rational Transductor
achieves \emph{100\% accuracy} across all lengths up to $L =
1000$. This confirms that by providing a diverse set of dynamic
kernels (here, a Stochastic head), the model can autonomously learn
and execute complex algorithmic rules that require both infinite
memory conservation and discrete state switching.

\subsection{Quantitative Generalization and Precision}
\label{sec:quantitative_exp}

While the previous examples (Parity, Addition) demonstrated the
ability to track discrete states, they did not test the capacity for
\textit{quantitative} accumulation over unbounded domains. As
established in \citep{CortesMohri2000}, weighted automata can
represent functions mapping sequences to numerical values (e.g.,
polynomial evaluation).

We consider the task of \emph{Base-2 Integer Evaluation}. The model
must map a binary string $x \in \{0, 1\}^L$ to its integer value
$f(x) = \sum_{t=1}^L x_t \cdot 2^{L-t}$. This requires the hidden
state to grow exponentially with sequence length
($v_t = 2 v_{t-1} + x_t$), serving as a stress test for the
``Linearity vs. Saturation'' hypothesis.

\paragraph{Experimental Setup.} We trained models on sequences of
length $L=64$. To strictly isolate architectural limitations from
hardware precision limits, all operations were performed in
\emph{Double Precision (Float64)}.  To fit within the dynamic range of
double precision floating point arithmetic, target integer values were
normalized to the unit interval $[0, 1]$ via scaling by $2^{-L}$.
We note that for sequence lengths $L > 53$, the integer values $2^L$
exceed the 53-bit significand precision of IEEE 754 Double Precision
(Float64). The reported "machine precision" MSE in
Figure~\ref{fig:arithmetic_eval} reflects the achieved numerical
precision of the model under double-precision arithmetic.
We compared the Rational Transductor against a Transformer and an
LSTM:
\begin{itemize}
\item Rational Transductor: We used a general Affine WFA
  parameterization (1 layer, $d=12$). Unlike the structured heads in
  Section~\ref{subsec:exp_addition}, here the model learns
  unconstrained transition matrices
  $\bA_0, \bA_1 \in \Rset^{d \times d}$ and bias vectors
  $\bb_0, \bb_1 \in \Rset^d$ from scratch, implementing the
  update $h_t = \bA_{x_t} h_{t-1} + \bb_{x_t}$. Initialization is set
  near identity to facilitate gradient flow.
  While a specialized scaling head could also solve this task, we used
  general matrices to demonstrate that the architecture can discover
  expansive dynamics ($|\lambda| > 1$) without explicit engineering.

\item Transformer: 3 layers, $d_{model}=32$, 4 heads, relative
  positional encoding.

\item LSTM: 1 layer, hidden dimension $d=32$.
\end{itemize}
Despite the Transductor being significantly smaller in parameters
($<1\text{k}$ vs $\approx 12\text{k}$), it is the only model
theoretically capable of unbounded linear growth.

\begin{figure}[H]
  \centering
  \includegraphics[width=0.4\linewidth]{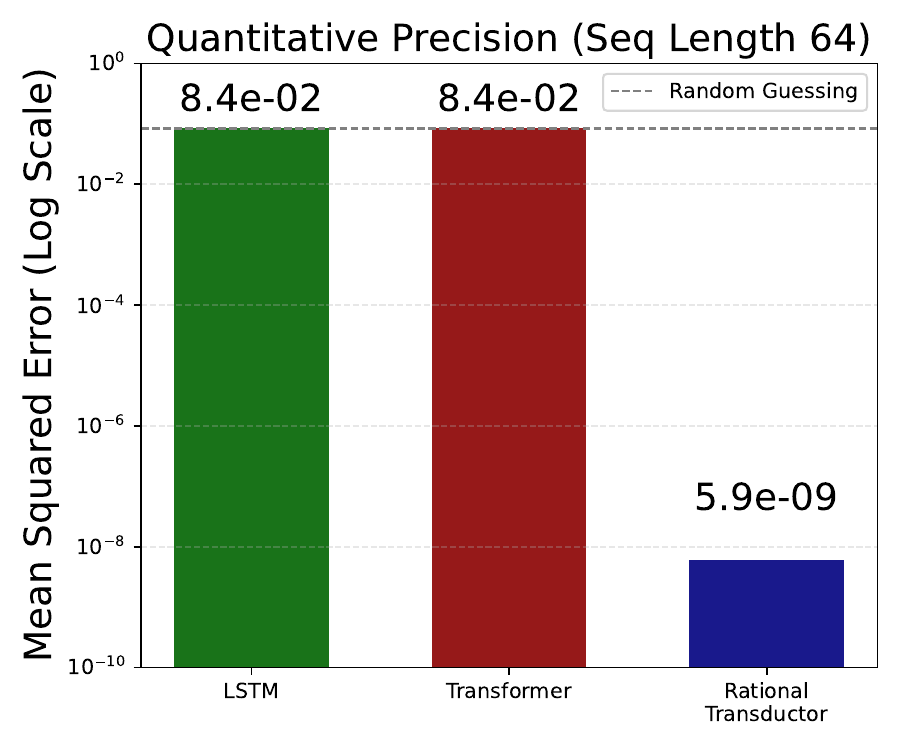}
  \caption{Quantitative Precision (Base-2 Evaluation). We trained
    models to compute the integer value of binary strings of length
    $L=64$ using double precision. Standard architectures (LSTM,
    Transformer) fail completely ($MSE \approx 8.4 \times 10^{-2}$),
    collapsing to the random-guessing baseline. The Rational
    Transductor learns the exact affine recurrence, achieving
    near-perfect precision ($MSE \approx 5.9 \times 10^{-9}$).}
    \label{fig:arithmetic_eval}
\end{figure}

\paragraph{Results.} As shown in Figure \ref{fig:arithmetic_eval},
this task reveals a sharp expressivity gap. Both the LSTM and
Transformer fail completely, converging to the variance of the dataset
($MSE \approx 8.4 \times 10^{-2}$). For the LSTM, the gradient signal
vanishes through 64 layers of saturating non-linearities ($\tanh$);
for the Transformer, the attention mechanism cannot resolve positional
weights spanning 19 orders of magnitude
($2^{64} \approx 1.8 \times 10^{19}$) amidst softmax noise.

In contrast, the Rational Transductor exploits its linear recurrence
to propagate gradients without attenuation, learning the exact Horner
scheme to high numerical precision ($MSE \approx 5.9 \times
10^{-9}$). We attribute this error floor primarily to optimization
convergence limits and accumulated algorithmic round-off noise rather
than the theoretical mantissa limits of the IEEE 754 Float64 standard.

\section{Conclusion}

We introduced \emph{Rational Transductors}, a hybrid architecture that bridges
the semantic flexibility of Transformers with the rigid state-tracking of formal
languages. By augmenting self-attention with a linear, matrix-valued
recurrence---parameterized via our novel Scaled Cayley mechanism---we overcome the
expressivity and learnability barriers inherent to standard Transformers. Our
theoretical analysis aligns this framework with the Krohn-Rhodes decomposition,
demonstrating that Transductors structurally capture the full hierarchy of
regular languages, including the cyclic groups that $\AC^0$ and $\TC^0$ models
fail to represent uniformly. Empirically, this yields solved length
generalization: Rational Transductors achieve perfect accuracy on tasks like
parity and modular addition where standard models fail, offering a rigorous path
toward neuro-symbolic architectures that unify connectionist learning with
algebraic reasoning.  In particular, Rational Transductors generalize perfectly
to $25\times$ the training horizon and solve 1000-digit addition with $100\%$
accuracy, while maintaining $O(\log T)$ parallel efficiency.

Furthermore, the Universal Transductor leverages
a mixture of orthogonal and stochastic kernels to autonomously select
the optimal dynamical bias for the task. Our theoretical analysis
shows that this architecture constitutes a neural instantiation of the
fundamental trinities of formal language theory:

\begin{itemize}
\item \textbf{Algebraic Completeness:} Via the Krohn-Rhodes
  decomposition, the Rational Head implements the missing
  \emph{group-theoretic} components (cyclic counters), while the
  Transformer implements the \emph{aperiodic} components (thresholds
  and resets), fulfilling the division of labor required for universal
  regular processing.

\item \textbf{Logical Completeness:} The model bridges the gap between
  the limitations of First-Order Logic ($\FO[<]$)
  (characteristic of hard-attention Transformers) and the full
  expressivity of Weighted MSO, enabling precise symbolic reasoning
  without sacrificing learnability.

\item \textbf{Statistical Generalization:} We proved that the model's
  reliability is guaranteed by a dual mechanism: \emph{contractive
    stability} for fading-memory tasks, and \emph{algebraic exactness}
  (via orthogonal parameterization) for infinite-memory tasks like
  Parity, ensuring robust performance independent of sequence length.
\end{itemize}

\textbf{Discussion and Scope.}
This paper establishes a \emph{foundational theory} for hybrid architectures
that combine attention with linear recurrence.  Our experimental validation
focuses on synthetic tasks chosen to isolate the key theoretical properties
(algebraic completeness, length generalization, state tracking).  While we do
not include large-scale NLP pretraining experiments, we emphasize that the
architecture imposes \emph{no constraints} on the Transformer backbone: the
rational head is a modular addition that can be integrated into any existing
Transformer.  Scaling to full pretraining, in the spirit of Jamba
\citep{lieber2024jamba} or Mamba-3 \citep{dao2025mamba3}, is the immediate next
step.  Beyond sequence modeling, we anticipate that the Rational Head's ability
to learn exact group-theoretic operations will prove useful in program
synthesis, formal verification, and any domain where rigid rule following must
coexist with learned semantic representations.  The combination of algebraic
interpretability and gradient-friendly parameterization positions Rational
Transductors as a principled building block for next-generation neuro-symbolic
systems.

\textbf{Limitations.}
The expressivity of the rational head is bounded by its state dimension $d$,
which corresponds to the rank of the underlying Hankel matrix, the
\emph{Algebraic Capacity} of the model.  Recognizing regular languages whose
minimal DFA has $n$ states requires $d \ge n$, and the parallel scan cost scales
as $O(d^2 T)$.  Consequently, very large-state automata may demand either large
$d$ (increasing compute) or architectural innovations (e.g., the stacked
cascading of Section~\ref{subsec:architectural_extensions}).  Exploring this
capacity-complexity trade-off in practical settings is an important direction
for future work.

\newpage
\section*{Acknowledgments}

I thank Corinna Cortes and Will Merrill for very helpful comments on
earlier drafts of this paper.

\bibliography{rbt}
\bibliographystyle{abbrvnat}

\clearpage

\newpage
\appendix

\section{Theoretical Background: Weighted Automata and
  Rational Power  Series}
\label{app:theoretical_background}

In this appendix, we place the architecture of Rational Transductors
within the broader theoretical framework of Weighted Finite Automata
(WFAs) \citep{mohri2009weighted} and Rational Power Series
\citep{SalomaaSoittola1978,BerstelReutenauer1988,KuichSalomaa1986}.
We define the specific class of series computed by our model and
outline the fundamental theorems that guarantee their expressivity and
learnability.

\subsection{Rational Power Series over a Field}

Let $\Sigma$ be a finite alphabet and $\Sigma^*$ be the free monoid
generated by $\Sigma$. A formal power series $S$ with coefficients in
the field of real numbers $\Rset$ is a mapping
$S\colon \Sigma^* \to \Rset$. The value of $S$ on a sequence
$x \in \Sigma^*$ is denoted by $(S, x)$. The set of all such formal
power series is denoted by $\Rset\llangle \Sigma^* \rrangle$. The
subset of \emph{Rational Power Series}, denoted
$\Rset^{\text{rat}}\llangle \Sigma^* \rrangle$, is the smallest
subalgebra of $\Rset\llangle \Sigma^* \rrangle$ containing all
polynomials (series with finite support) that is closed under the
following rational operations:
\begin{itemize}

\item \textbf{Sum:} $(S + T, x) = (S, x) + (T, x)$;

\item \textbf{Cauchy Product:}
  $(S \cdot T, x) = \sum_{uv=x} (S, u)(T, v)$

\item \textbf{Kleene Star (Closure):} $S^* = \sum_{n=0}^\infty S^n$,
  provided $(S, \e) = 0$ to ensure convergence.

\end{itemize}

\subsection{Linear Representations and WFAs}

A fundamental result in the theory of weighted automata is the
Sch\"utzenberger representation theorem
\citep{Schutzenberger1961}, which establishes that rational
series are exactly those recognizable by finite weighted automata.
\begin{definition}[Linear Representation]
  A linear representation of dimension $d$ over $\Rset$ is a triple
  $(\balpha, \{\sfM_\sigma\}_{\sigma \in \Sigma}, \bbeta)$, where:
\begin{itemize}
\item $\balpha \in \Rset^d$ is the initial weight vector (column).
\item $\sfM_\sigma \in \Rset^{d \times d}$ are the transition matrices
  for each $\sigma \in \Sigma$.
\item $\bbeta \in \Rset^d$ is the final weight vector (column).
\end{itemize}
\end{definition}

This representation computes a series $S$ defined by:
\begin{equation}
  (S, x) = \bbeta^\top \sfM_{x_T} \dots \sfM_{x_1} \balpha.
\end{equation}

\subsection{Fundamental Results}

\paragraph{The Hankel Matrix and Fliess' Theorem.}
A central tool in the analysis of rational series is the
\emph{Hankel matrix} $H_S$, an infinite matrix indexed by pairs of
strings $(u, v) \in \Sigma^* \times \Sigma^*$, where the entry at
$(u, v)$ is $(S, uv)$.
\begin{theorem}[\citep{Fliess1974}]
  A series $S$ is rational if and only if its Hankel matrix $H_S$ has
  finite rank. Furthermore, the rank of $H_S$ is equal to the
  dimension $d_{\min}$ of the minimal linear representation of $S$.
\end{theorem}

\paragraph{Minimization and Learning.}
For any rational series, the minimal linear representation is unique
up to a similarity transformation (see \citep{BalleMohri2015} for a
short proof and illustration). Furthermore, spectral learning
frameworks \citep{balle2012spectral} demonstrate that rational series
can be learned efficiently under specific rank conditions via the
singular value decomposition of the Hankel matrix.

\subsection{Remarks on Terminology and Graph Interpretation}

To clarify the relationship between our linear algebraic definition
and standard automata theory, we provide the following remarks.

\textbf{Equivalence of Dimension and State Count.}
The definition of a WFA via a linear representation of dimension $d$
is mathematically isomorphic to a Weighted Finite Automaton with
exactly $d$ states.
\begin{itemize}
\item The dimension $d$ corresponds to the set of discrete states
  $Q = \{q_1, \dots, q_d\}$.
\item The entry $(\sfM_\sigma)_{ij}$ corresponds to the weight of the
  edge transitioning from state $q_j$ to state $q_i$ upon reading
  symbol $\sigma$.
\item Consequently, Fliess' theorem can be equivalently stated as: the
  rank of the Hankel matrix equals the number of states in the minimal
  WFA recognizing the series.
\end{itemize}

\textbf{The State Vector as a Weight Distribution.}
In the context of deep learning, the vector $h_t \in \Rset^d$ is often
referred to as the ``hidden state.'' In the automata theoretic view,
this vector represents the \textit{distribution of accumulated
  weights} over the $d$ states of the automaton at time $t$.
Specifically, the $i$-th component $h_{t,i}$ is the sum of weights of
all paths in the automaton ending at state $q_i$ given the input
prefix $x_{1:t}$.

\textbf{Deterministic Computation in Vector Space.}  While Weighted
Finite Automata are not generally determinizable in the graph sense
(i.e., transforming into an equivalent WFA with only one non-zero path
per string) \citep{mohri2009weighted}, the linear recurrence
$\bh_t = \sfM_{x_t} \bh_{t-1}$ constitutes a deterministic update in
the vector space $\Rset^d$. This ensures that the Rational Transductor
architecture remains deterministic and efficient to compute, despite
the underlying WFA potentially representing non-deterministic weighted
paths.

Finally, WFAs have been successfully used in a variety of
applications, including speech recognition
\citep*{MohriPereiraRiley2002}. The OpenFST software library
\citep*{AllauzenRileySchalkwykSkutMohri2007} provides a very general
and efficient implementation of the representation and algorithms
related to WFAs.

\newpage
\section{Experimental Details and Hyperparameters}
\label{app:hyperparameters}

\subsection{Hyperparameter Specifications}

To ensure reproducibility, we detail the exact hyperparameters used
for the experiments in Section~\ref{sec:experiments}. All models were
implemented in PyTorch and trained on a single NVIDIA T4 or A100
GPU. Optimization was performed using AdamW
\citep{loshchilov2017decoupled} or Adam \citep{kingma2014adam}.

Table \ref{tab:hyperparams} summarizes the configurations for all four
synthetic tasks. Note that for the \emph{Long-Integer Addition}
task, we used a curriculum learning strategy where the training
sequence length was sampled uniformly from $U[10, 40]$ at each step to
encourage robust generalization.

\begin{table}[h]
\centering
\caption{Hyperparameters for Rational Transductor Experiments. (RT:
  Rational Transductor, TF: Transformer.)}
\label{tab:hyperparams}
\resizebox{\textwidth}{!}{%
\begin{tabular}{lcccc}
\toprule
  \textbf{Config / Task} & \textbf{Modulo Counting} & \textbf{Length Gen.}
  & \textbf{Long Addition} & \textbf{Base-2 Eval} \\
 & (Sec. 7.1) & (Sec. 7.2) & (Sec. 7.4) & (Sec. 7.5) \\
\midrule
\multicolumn{5}{l}{\textit{Model Architecture}} \\
Hidden Dim ($d_{\text{model}}$) & 32 & 32 & 32 & 12 (Rational Transductor) / 32 (Transformer) \\
Rational State Dim ($d_{\text{rat}}$) & 8 & 8 & 4 & 12 \\
Layers & 2 & 2 & 2 & 1 (Rational Transductor) / 3 (Transformer) \\
Heads & 4 & 4 & 4 & 4 \\
  Parameterization & Orthogonal (Cayley) & Orthogonal (Cayley) & Stochastic (Softmax)
                           & Affine (General) \\
\midrule
\multicolumn{5}{l}{\textit{Optimization}} \\
Seq Length ($L_{\text{train}}$) & 50 & 40 & $U[10, 40]$ & 64 \\
Batch Size & 64 & 64 & 64 & 32 \\
Optimizer & AdamW & AdamW & AdamW & Adam \\
  Learning Rate & $5 \times 10^{-3}$ & $5 \times 10^{-3}$ & $5 \times 10^{-3}$
                           & $1 \times 10^{-2}$ \\
Scheduler & None & Cosine Annealing & None & Cosine Annealing \\
Gradient Clip & 1.0 & 1.0 & 1.0 & 1.0 \\
Training Steps & 3,000 Steps & 3,000 Steps & 4,000 Steps & 3,600 Steps (60 Epochs) \\
Loss Function & Cross Entropy & Cross Entropy & Cross Entropy & MSE \\
Precision & Float32 & Float32 & Float32 & Float64 \\
\bottomrule
\end{tabular}
}
\end{table}

\subsection{Statistical Significance and Stability}

\textbf{Stability of Initialization:} The Near-Identity Initialization
described in Section~\ref{subsec:initialization} provides a highly
stable starting point for learning algebraic structures. In our
experiments (specifically Subsection~\ref{subsec:exp_regular_gap} and
Subsection~\ref{subsec:exp_length_gen}), the Rational Transductor
converged to 100\% training accuracy within the first 500-1000 steps
in every trial. The loss curves exhibit a deterministic drop
characteristic of solving a convex-like problem in the lifted state
space, rather than the high-variance grokking often seen in standard
Transformers on these tasks.

\textbf{Variance:} We repeated all synthetic experiments across $N=5$
random seeds to ensure statistical significance.
\begin{itemize}
  
\item \textbf{Rational Transductor:} Achieved 100\% accuracy on the
  training distribution in 5/5 runs for Modulo Counting, Length
  Generalization, and Long-Integer Addition. For the Long-Integer
  Addition task, sequence-level standard deviation was negligible
  ($<0.05\%$). The Base-2 Evaluation task exhibited high stability
  across trials with an MSE standard deviation of
  $< 1.0 \times 10^{-10}$.

\item \textbf{Baseline Transformer:} Consistently failed to generalize
  beyond $L_{\text{train}}$, with accuracy collapsing to the random
  baseline (20\%) in all seeds, exhibiting negligible variance in its
  failure mode.
\end{itemize}
We conclude that the reported performance gap is due to the
fundamental difference in inductive bias (Recurrent vs. Attention),
not initialization luck.

\end{document}